\pdfoutput=1

\documentclass[11pt]{article}

\PassOptionsToPackage{dvipsnames,usename}{xcolor}
\usepackage{acl}

\usepackage{times}
\usepackage{latexsym}

\usepackage[T1]{fontenc}

\usepackage[utf8]{inputenc}

\usepackage{microtype}

\usepackage{inconsolata}
\usepackage{graphicx}
\usepackage{amsmath}
\usepackage{amssymb}
\usepackage{amsthm}
\usepackage{amsfonts}
\usepackage{bbm}
\usepackage{xspace}
\usepackage{pbox}
\usepackage{booktabs}
\usepackage{hyperref}
\hypersetup{
    colorlinks=true,
    linkcolor=darkblue,
    filecolor=darkblue,      
    urlcolor=darkblue,
}

\usepackage{makecell}
\usepackage{scalerel}

\usepackage{cleveref}
\crefname{section}{\S}{\S\S}
\Crefname{section}{\S}{\S\S}
\crefformat{section}{\S#2#1#3}
\crefname{figure}{Fig.}{Fig.}
\crefname{alg}{Alg.}{Alg.}
\crefname{line}{line}{lines}
\crefname{appendix}{App.}{}
\crefname{equation}{Eq.}{Eqs.}
\crefname{table}{Table}{Tables}
\crefname{proposition}{Proposition}{Propositions}

\newtheorem{definition}{Definition}
\newtheorem{lemma}{Lemma}
\newtheorem{takeaway}{Takeaway}

\usepackage[textsize=tiny]{todonotes}

\newcommand{\colourword}{purple}
\newcommand{\colourdubword}{MidnightBlue}

\newcommand{\T}{\mathrm{T}}

\newcommand{\mywordfunc}[2]{\newcommand{#1}{{\color{\colourword}#2}}}
\newcommand{\mydubwordfunc}[2]{\newcommand{#1}{{\color{\colourdubword}#2}}}

\mywordfunc{\alphabet}{\Sigma}
\mywordfunc{\word}{w}
\mywordfunc{\words}{\boldsymbol{w}}
\mywordfunc{\Word}{\mathrm{W}}
\mywordfunc{\Words}{\boldsymbol{\Word}}
\mywordfunc{\wordsspace}{\mathcal{W}}

\mydubwordfunc{\Sigmaduplicated}{\overline{\Sigma}}
\mydubwordfunc{\dubword}{c}
\mydubwordfunc{\dubwords}{\boldsymbol{\dubword}}
\mydubwordfunc{\Dubword}{\mathrm{C}}
\mydubwordfunc{\Dubwords}{\boldsymbol{\Dubword}}
\mydubwordfunc{\dubwordsspace}{\mathcal{C}}
\mydubwordfunc{\alphabetdupsmap}{\mathbb{S}}

\newcommand{\alphabetdup}{\mathcal{S}}
\newcommand{\alphabetdupd}{\overline{\mathcal{S}}}

\newcommand{\pdub}{p}
\newcommand{\phybrid}{p_{\alphabetdupsmap}}

\newcommand{\model}{\widehat{p}}
\newcommand{\modeldub}{\widehat{p}}
\newcommand{\modelproj}{\model_\alphabetdupsmap}

\newcommand{\whitespacedupsmap}{\alphabetdupsmap_\text{space}}

\newcommand{\pluraldupsmap}{\alphabetdupsmap_\text{plural}}

\newcommand{\lowerdupsmap}{\alphabetdupsmap_\text{lower}}

\newcommand{\alldupsmap}{\alphabetdupsmap_\text{all}}

\newcommand*{\circled}[1]{\tikz[baseline=(char.base)]{
            \node[shape=circle,draw,inner sep=1pt] (char) {\normalfont{\small #1}};}}

\newcommand{\subword}[1]{\textit{#1}}
\newcommand{\loss}{\mathcal{L}}
\newcommand{\dataset}{\mathcal{D}}
\newcommand{\trainset}{\dataset_{\mathrm{trn}}}
\newcommand{\testset}{\dataset_{\mathrm{eval}}}

\newcommand{\defeq}{\mathrel{\stackrel{\textnormal{\tiny def}}{=}}}

\newcommand{\ent}{\mathrm{H}}
\newcommand{\xent}{\widehat{\ent}}

\newcommand{\mi}{\mathrm{MI}}

\newcommand{\alphabetdupnumber}[1]{\alphabetdup_{\scaleto{\circled{#1}}{7pt}}}

\newcommand{\ppl}{\mathrm{PPL}}
\newcommand{\pplbf}{\mathbf{PPL}}
\newcommand{\ppldedup}{\ppl_\alphabetdupsmap}
\newcommand{\ppldedupbf}{\pplbf_\alphabetdupsmap}
\newcommand{\lossdedup}{\loss_\alphabetdupsmap}

\DeclareRobustCommand{\dubwordnumber}[1]{\dubword_{\scaleto{\circled{#1}}{7pt}}}
\newcommand{\dubwordone}{\dubword_{\scaleto{\circled{1}}{7pt}}}
\newcommand{\dubwordtwo}{\dubword_{\scaleto{\circled{2}}{7pt}}}

\DeclareRobustCommand{\wordnumber}[1]{\word_{\scaleto{\circled{#1}}{7pt}}}
\DeclareRobustCommand{\wordnumberlarge}[1]{\word_{\scaleto{\circled{#1}}{10pt}}}
\newcommand{\wordone}{\word_{\scaleto{\circled{1}}{7pt}}}
\newcommand{\wordtwo}{\word_{\scaleto{\circled{2}}{7pt}}}
\newcommand{\wordthree}{\word_{\scaleto{\circled{3}}{7pt}}}
\newcommand{\wordfour}{\word_{\scaleto{\circled{4}}{7pt}}}

\newcommand{\prduplicate}{p(\wordnumber{i}' \mid \dubwordnumber{i})}
\newcommand{\prdontduplicate}{p(\wordnumber{i} \mid \dubwordnumber{i})}

\newcommand{\flatten}{\mathrm{flatten}}
\newcommand{\one}{\mathbbm{1}}

\newcommand{\defn}[1]{\textbf{#1}}

\newcommand{\saveForCr}[1]{}

\title{On the Effect of (Near) Duplicate Subwords in Language Modelling}

\newcommand*\samethanks[1][\value{footnote}]{\color{darkblue} \footnotemark[#1]}

\newcommand{\emailadress}[1]{\texttt{#1}}

\newcommand{\ethz}{1}
\newcommand{\ethai}{2}

\author{Anton Sch\"afer$^\ethz$, Thomas Hofmann$^\ethz$, Imanol Schlag$^{\ethai,}$\thanks{Shared supervision.}, Tiago Pimentel$^{\ethz,}$\samethanks\\
  $^\ethz$ETH Z\"urich, $^\ethai$ETH AI Center \\
   \emailadress{scanton@ethz.ch}, \{\emailadress{thomas.hofmann}, \emailadress{imanol.schlag}, \emailadress{tiago.pimentel}\}\texttt{@inf.ethz.ch}
}

\begin{document}
\maketitle
\begin{abstract}
Tokenisation is a core part of language models (LMs).
It involves splitting a character sequence into subwords which are assigned arbitrary indices before being served to the LM.
While typically lossless, however,
this process may lead to less sample efficient LM training: as it removes character-level information, it could make it harder for LMs to generalise across similar subwords, such as \subword{now} and \subword{Now}. 
We refer to such subwords as \defn{near duplicates}.
In this paper, we study the impact of near duplicate subwords on LM training efficiency.
First, we design an experiment that gives us an upper bound to how much we should expect a model to improve if we could perfectly generalise across near duplicates.
We do this by duplicating each subword in our LM's vocabulary, creating perfectly equivalent classes of subwords.
Experimentally, we find that LMs need roughly 17\% more data when trained in a fully duplicated setting.
Second, we investigate the impact of naturally occurring near duplicates on LMs.
Here, we see that merging them considerably hurts LM performance. 
Therefore, although subword duplication negatively impacts LM training efficiency, naturally occurring near duplicates may not be as similar as anticipated, limiting the potential for performance improvements.

\vspace{.3em}
\hspace{.5em}\includegraphics[width=1.25em,height=1.25em]{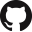}{\hspace{.75em}\parbox{\dimexpr\linewidth-2\fboxsep-2\fboxrule}{\vspace{-5pt} \href{https://www.github.com/antonschafer/duplicate-subwords}{antonschafer/duplicate-subwords}}}
\end{abstract}

\section{Introduction}

Most modern language models (LMs) do not have direct access to the bits or characters which make up the text that they must model.
Rather, they operate on higher-level units, so-called tokens, which are elements of a finite set of previously defined \defn{subwords}.
This set of subwords is typically obtained as the output of a tokenisation method and forms an LM's \defn{vocabulary} \cite{gage1994new,sennrich-etal-2016-neural,kudo-2018-subword,wu2016google}.
Importantly, most tokenisation algorithms are lossless: the original character sequence is perfectly recoverable from its tokenised version.\looseness=-1

\begin{table}
\centering
\begin{tabular}{lcc}
\toprule
\textbf{Model} & \textbf{Near Duplicate Rate} \\
\midrule
GPT-\{3.5, 4, 4-turbo\} & 43\%\\
Claude 2.1 & 46\% \\
Llama 1 \& 2 & 35\% \\
Mistral 7B \& 8x7B &  37\%  \\
Gemma 7B & 39\% \\
\bottomrule
\end{tabular}
\caption{Near duplicate rates of modern LLM vocabularies. We consider subwords that only differ in whitespace, capitalization, or plural suffix as equivalent ($\alldupsmap$ deduplication mapping, see \cref{sec:deduplication_implementation}). For details, see \cref{sec:llm_dup_stats}.}
\label{tab:llm_dup_stats_short}
\vspace{-7pt}
\end{table}

A language model's vocabulary, however, may contain several \defn{near duplicate subwords}:
minimal pairs like \subword{now} and \subword{Now}, with roughly the same semantic meaning, but which differ due to typos, whitespace marking, or capitalisation \cite{stanic2023languini}.
Such near duplicates can make up over 40\% of the vocabulary of modern LMs (see \cref{tab:llm_dup_stats_short}).
Intuitively, if the model had access to character-level information, it should trivially generalise what it learns from one of these forms to the other.
Given only access to subword-level inputs, however, the model may not be able to do the same, or may require more data to do so.\looseness=-1

Previous work tried to address this issue by 
modelling language directly at the character, byte, or even pixel level \cite[\textit{inter alia}]{kim2016character,
clark-etal-2022-canine, xue-etal-2022-byt5,yu2023megabyte,rust2023language}.\footnote{These works are not solely motivated by near duplicates. 
Other commonly named advantages of character/byte-level models include: the possibility of optimising them end-to-end without relying on a two-stage approach, greater flexibility by not committing to a (potentially suboptimal) tokeniser, 
and more direct access to word forms which might be relevant in, e.g.,  word-play related tasks \citep{rozner2021decrypting}.}
However, while the existing literature has proposed several solutions to improve LMs generalisation across near duplicates, a proper quantification of the issue is still lacking.

In this work, we thus take a step back and assess the actual impact of near duplicate subwords on LMs' performance. 
To this end, we first propose a controlled synthetic setting where we duplicate every subword in our LM's vocabulary, allowing us to  \textbf{(1.)} quantify an LM's ability to generalise across perfectly equivalent duplicates and to \textbf{(2.)} carefully investigate how the generalisation happens.
This yields an upper bound on the cost incurred due to limited generalisation across real near duplicates: a vocabulary with 40\% duplicates (common for LLMs, see \cref{tab:llm_dup_stats_short}) may reduce data efficiency by up to 10\%.
We then \textbf{(3.)} investigate the tightness of this upper bound by merging naturally occurring near duplicates in an LM's vocabulary.
We find that deduplicating the vocabulary in this way, in general, hurts performance instead of improving it.
This suggests that real near duplicates might be less similar than anticipated, which might impose challenges when trying to leverage their similarity to improve LMs' performance.

\section{Language Modelling}
\label{sec:lang_modelling}

\newcommand{\eos}{\texttt{eos}\xspace}

Let $\alphabet$ be a vocabulary of subwords.
A language model $\model$ is formally defined as a probability distribution over the set of all finite sequences of subwords $\words = (\word_1, \word_2, ...) \in \wordsspace \defeq \alphabet^{*}$:\footnote{We define $\alphabet$ with a special end-of-sequence symbol (\eos). Any string with mid-sequence \eos is assigned probability zero.\looseness=-1}
\begin{align}
    \model(\words) = \prod_{t=1}^{|\words|} \model(\word_t \mid \words_{<t})
\end{align}
where $ \model(\word_t \mid \words_{<t})$ is the probability of token $\word_t$ given the sequence of previous tokens $\word_0, ... \word_{t-1}$.

In order for $\model$ to approximate the true distributions over natural strings $p(\words)$, we train this model to minimise its cross-entropy with $p$:
\begin{align}
    \xent(\Words) = \sum_{\words \in \wordsspace} p(\words) \sum_{t=1}^{|\words|} \log \frac{1}{\model(\word_t \mid \words_{<t})}
\end{align}
where $\Words$ represents a $\wordsspace$-valued random variable.
Since we do not know $p$, we approximate this objective using a finite training set $\trainset = \{\words_n\}_{n=1}^{N} \sim p(\words)$. 
This leads to the loss function:
\begin{align}
    \loss(\model(\words), \trainset) = \frac{1}{N} \!\sum_{\words \in \trainset}\! \sum_{t=1}^{|\words|} \log \frac{1}{\model(\word_t \mid \words_{<t})}
\end{align}

\section{Subword Duplication}
\label{sec:subword_duplication}

Now, let's assume there exist in our alphabet pairs or groups of subwords which are nearly identical---in both their orthography and semantics.
Such near duplicates can arise from various sources, including but not limited to 
capitalization differences (e.g., \subword{now} vs. \subword{Now}), 
typographical errors (\subword{language} vs. \subword{langauge}), 
the presence or absence of whitespace (e.g., \subword{the} vs. \subword{\_the}),\footnote{We denote spaces as \subword{`\_'} for readability.} 
and variations in spelling (e.g., \subword{modeling} vs. \subword{modelling}).
The question we are concerned with is: how might such subword duplication affect the performance of language models?\looseness=-1

To define this question formally, let $\alphabetdup$ be a set of disjoint sets of near duplicate subwords:
\begin{align}
    \alphabetdup = \left\{\begin{array}{c}
         \{\wordone, \wordtwo, \wordthree, \wordfour\}, \\
         \{\wordnumber{5}, \wordnumber{6}\}, \cdots, \{\wordnumber{i}, \wordnumberlarge{i+1}\}
    \end{array}\right\}
\end{align}
We index these as $\alphabetdupnumber{i}$ to represent the i'th set of near duplicates.
Further, let $\alphabetdupd = \{\dubwordnumber{i} \mid 0 < i \leq |\alphabetdup|\}$ be a set of \defn{canonical symbols} which we will use to represent each of these duplicate sets.
To find out how duplication affects LMs, we can create a map $\alphabetdupsmap: \alphabet \to \Sigmaduplicated$ which deduplicates subwords, mapping duplicates to the corresponding canonical symbols. 
This map is defined as:\looseness=-1
\begin{align}
    \alphabetdupsmap(\word) = \left\{\begin{array}{lr}
         \word & \mathrm{if }\,\, \word \notin \flatten(\alphabetdup) \\
         \dubwordnumber{i} & \mathrm{if }\,\, \word \in \alphabetdupnumber{i}
    \end{array}\right.
\end{align}
with $\Sigmaduplicated \defeq (\alphabet \setminus \flatten(\alphabetdup)) \cup \alphabetdupd $.
We are now in a position to define a distribution over deduplicated subword sequences:
\begin{align}
    \pdub(\dubwords) = 
    \sum_{\words \in \wordsspace} 
    p(\words)\,\one\{\dubwords = \alphabetdupsmap(\words)\}
    \label{eq:pdub_def}
\end{align}
where we overload $\alphabetdupsmap$ to operate on sequences $\words$ by applying it elementwise on each subword $\word_t \in \words$.
Note that $\pdub(\dubwords)$ is now a distribution over deduplicated subword sequences $\dubwords = (\dubword_1, \dubword_2, ...) \in \dubwordsspace \defeq \Sigmaduplicated^{*}$.
We can further define a hybrid conditional \defn{projected distribution} as:
\begin{align}
    &\phybrid(\dubword_t \mid \words_{<t}) = \label{eq:pproj_def} \\
    &\qquad\qquad \sum_{\word \in \alphabet} p(\word \mid \words_{<t})\,\one\{\dubword_t = \alphabetdupsmap(\word)\} \nonumber
\end{align}
In words, the conditional probability $\phybrid(\dubword_t \mid \words_{<t})$ of a deduplicated subword $\dubword_t$ is defined as the sum of the conditional probabilities of all subwords which map to it through $\alphabetdupsmap(\word)$.

\subsection{Comparing (De-)Duplicated LMs}
\label{sec:comparing_deduplicated_lms}

Typically, LMs are evaluated based on their cross-entropy (or perplexity) on a held-out test set.
It would, however, be unfair to simply compare the cross-entropies of LMs trained on $\pdub(\dubwords)$ and $p(\words)$.
The cross-entropy is lower bounded by the entropy---with a perfect model's cross-entropy equalling the entropy.
If $\pdub(\dubwords)$'s and $p(\words)$'s entropies are different, they would impose different optima achievable by LMs trained in each setting.
We write these distributions' entropies as:
\begin{align}
    \ent(\Words) 
    &= \sum_{\words} p(\words) \sum_{t=1}^{|\words|} \log \frac{1}{p(\word_t \mid \words_{<t})}
\end{align}
and $\ent(\Dubwords)$, analogously, 
where $\Dubwords$ denotes a $\dubwordsspace$-valued random variable.
Now note that, given their definition in Eq.~\ref{eq:pdub_def}, deduplicated sequences $\Dubwords$ are deterministic given the original subwords $\Words$.
The two entropies above are thus related via equation:\looseness=-1
\begin{align}
    \ent(\Words) = \ent(\Dubwords) + \ent(\Words \mid \Dubwords) \label{eq:factorize_ent}
\end{align}
Assuming $\Words$ cannot be deterministically predicted from $\Dubwords$, we have that $\ent(\Words \mid \Dubwords) > 0$;
this implies that predicting $\Words$ is strictly harder than $\Dubwords$.\looseness-1

To make these settings more easily comparable, we define $p(\words)$'s \defn{projected entropy} as:
\begin{align}
    \ent_{\alphabetdupsmap}(\Words) 
    &= \sum_{\words \in \wordsspace} p(\words) \sum_{t=1}^{|\words|} \log \frac{1}{\phybrid(\alphabetdupsmap(\word_t) \mid \words_{<t})} 
\end{align}
This entropy measures the uncertainty in predicting a deduplicated subword $\dubword_t$ given the duplicated context $\words_{<t}$.
Interestingly, we can show that:
\begin{align} \label{eq:MI_similarity}
    \ent_{\alphabetdupsmap}(\Words) &= \ent(\Dubwords) - \mi(\Words_{<\T}; \Dubwords_{\T} \mid \Dubwords_{< \T})
\end{align}
where $\mi(\Words_{<\T}; \Dubwords_{\T} \mid \Dubwords_{< \T})$ is the mutual information between a  subword context $\words_{<t}$ and the next deduplicated token $\dubword_t$ conditioned on the previous deduplicated tokens $\dubwords_{<t}$ (see \cref{app:MI_similarity} for a proof and the precise definition of this mutual information).
In both \cref{sec:results_perfect_duplication} and \cref{sec:results_natural_duplication}, we discuss how this value relates to our research question.

\section{Experimental Setup}

We implement all of our experiments in the codebase of the Languini Kitchen \citep{stanic2023languini}.
Below, we provide an overview of the models and datasets used here.
We refer the reader to \citeposs{stanic2023languini} work for more details regarding implementation choices, training setup, and dataset collection.

\paragraph{Model.}
We use the GPT model from Languini, which is a GPT-2 style transformer decoder \cite{radford2019language}. 
Unless otherwise noted, we use the ``small'' configuration with 12 layers, hidden size 768, and 85M non-embedding parameters.
We train models with sequence length 512, batch size 128, the Adam optimiser \cite{kingma2015adam}, and a cosine learning rate schedule from 6e-4 to 6e-6 with 500 warmup steps.

\paragraph{Data.}
We train on the Languini training data, a filtered version of the books3 subset from the Pile \cite{gao2020pile}.
For evaluation, we use the held-out Languini test set, which contains 11M tokens.
This data is pre-tokenised into a vocabulary of size 16k using a BPE tokeniser \cite{gage1994new,sennrich-etal-2016-neural} trained using SentencePiece \cite{kudo2018sentencepiece}.
Unless otherwise noted, we train our models for 18,265 steps---i.e., the first 1.2B tokens in our dataset---which corresponds to training the small GPT model for 6h on an RTX 3090 GPU; this is Languini's GPT small 6h setting.

\paragraph{Evaluation.}
We generally report our model's perplexity on the test set as our evaluation metric.
To ensure sufficient context for all predictions, we use a sliding window with steps of 128: we fill in a 512 tokens context, ignore the model's outputs on the initial 384, and evaluate it only using the last 128 tokens.
For models $\modeldub(\dubwords)$ trained on the deduplicated setting, we simply report their perplexities, defined as the exponentiated cross-entropy evaluated on a held-out test set $\testset$: 
$\ppl(\Dubwords) = \exp\left(\loss(\model(\dubwords); \testset)\right)$.
When evaluating models $\model(\words)$, trained in the duplicated setting,
we report their \defn{projected perplexity}, defined as:\looseness=-1
\begin{align}
    &\ppldedup(\Words) = \exp \left(\lossdedup(\model(\words); \testset)\right) \\
    &\quad = \exp\!\!\left(\!\frac{1}{N} \!\!\sum_{\words \in \testset}\! \sum_{t=1}^{|\words|} \log \frac{1}{\modelproj(\alphabetdupsmap(\word_t) \mid \words_{<t})} \right) \nonumber
\end{align}
where $\modelproj$ is defined analogously to $\phybrid$ (see Eq.~\ref{eq:pproj_def}).
Intuitively, we add up the probabilities of subwords $\word$ that are equivalent under $\alphabetdupsmap$ (i.e., which map to the same canonical symbol $\dubword$), to avoid giving $\modeldub(\dubwords)$ an unfair advantage over $\model(\words)$.

\section{Experiments and Results}

We evaluate our LM's ability to generalise over (near) duplicates in two settings: perfect and natural duplication.
In the perfect duplication setting, 
we compare LMs trained using either the default or a synthetically duplicated vocabulary; this gives us an upper bound for the impact of real near duplicates, as synthetically duplicated subwords are perfectly comparable in terms of their semantics.
In the natural duplication setting, we deduplicate the default vocabulary by merging real near duplicates. 
By comparing performance on the default and the deduplicated vocabulary, we verify whether the effect of natural near duplicates in LMs is comparable to the effect of perfect duplicates.

\begin{figure*}[t]
    \centering
    \includegraphics[width=.9\linewidth]{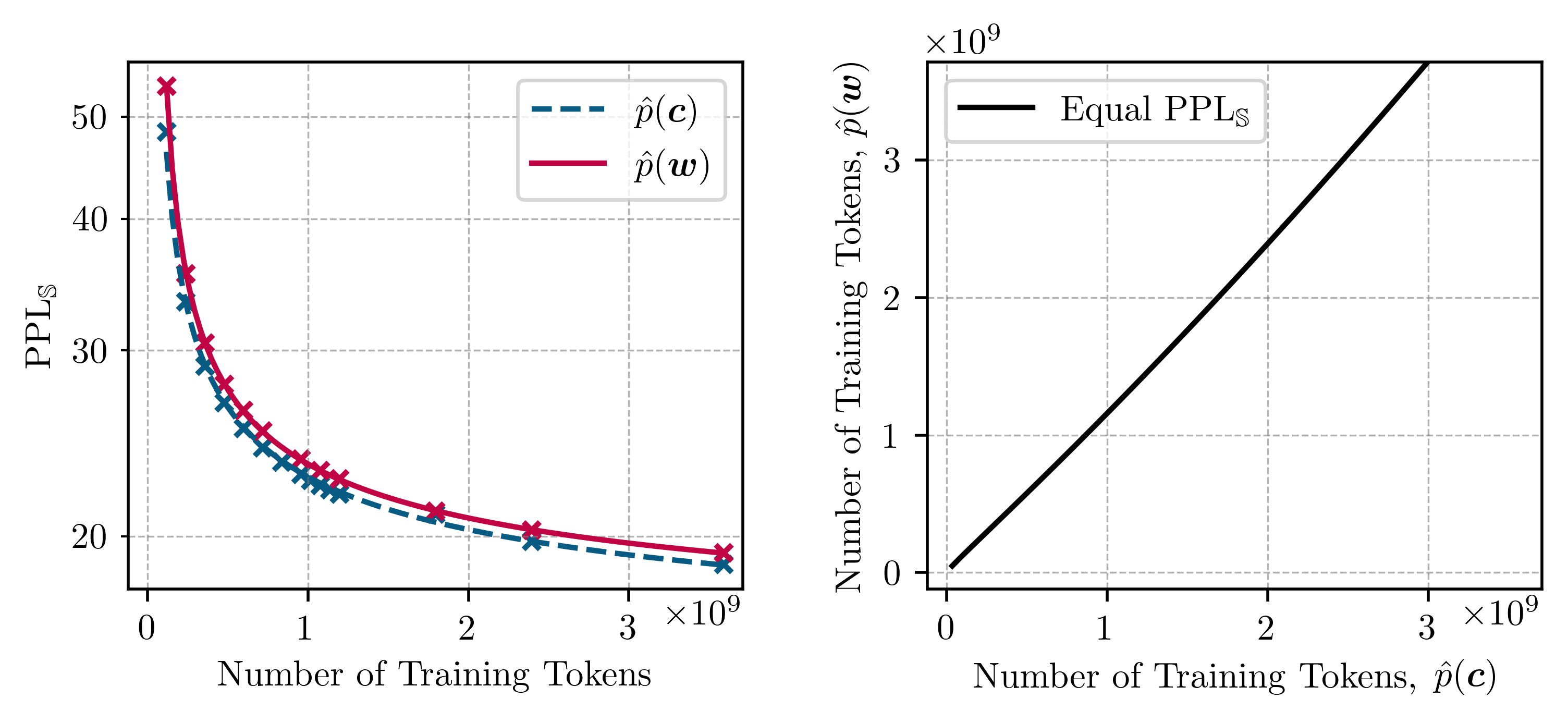}
    \vspace{-10pt}
    \caption{Left: Fitted power laws capturing the relationship between training data and $\ppldedup$. Our standard training set contains around $1.2$B tokens. Right: Data required to achieve the same performance with $p(\dubwords)$ and $p(\words)$, computed based on the fitted scaling law curves. In the considered interval, this curve's slope---which roughly corresponds to $\frac{\text{number of training tokens for $p(\words)$}}{\text{number of training tokens for $p(\dubwords)$}}$---is approximately equal to $\frac{1}{0.85}$.}
    \label{fig:scaling_dup}
\vspace{-5pt}
\end{figure*}

\subsection{Perfect Duplication}
\label{sec:results_perfect_duplication}

In this first set of experiments, we assume an idealised situation where all subwords in a near duplicate set $\word \in \alphabetdupnumber{i}$ are perfectly interchangeable with each other.
Choosing among these subwords, then, is neither impacted by prior subwords nor impacts future subword choices.
In this case, we can relate distributions $p(\words)$ and $\pdub(\dubwords)$ as:\looseness=-1
\begin{align}\label{eq:perfect_duplicate_choice}
    p(\words) 
    &= \prod_{t=1}^{|\words|} \pdub(\dubword_t \mid \dubwords_{<t})\!\!\underbrace{p(\word_t \mid \dubword_t)}_{\text{duplicate choice}}
\end{align}
Further, we can show that in this idealised setting $\mi(\Words_{<\T}; \Dubwords_{\T} \mid \Dubwords_{< \T}) = 0$; this is because the decomposition above implies conditional independence between $\word_t$ and any $\dubword_{t'}$ given $\dubword_t$. 
We thus have:
\begin{align}
    \ent_{\alphabetdupsmap}(\Words) &= \ent(\Dubwords)
\end{align}
which creates a perfectly controlled setting to evaluate language models.
If we could train a perfect language model on either distribution $p(\words)$ or $p(\dubwords)$, we would achieve the same performance in both settings.
Any difference in language modelling performance between these settings must thus derive from a language model's lack of ability to generalise from observing near duplicate subwords.

\begin{table}[t]
\centering
\begin{tabular}{lcc}
\toprule
& \multicolumn{2}{c}{$\ppldedupbf$} \\ \cmidrule(lr){2-3}
\textbf{Model}             & \textbf{GPT-S}   & \textbf{GPT-M}  \\
\midrule
$\model(\dubwords)$                   & 21.9      & 16.3         \\ 
$\model(\dubwords)$, 85\% of data    & 22.6       & 16.7       \\ 
$\model(\dubwords)$, 50\% of data    & 25.3       & -        \\ 
\midrule
$\model(\words)$               & 22.7  & 16.7          \\ 
\bottomrule
\end{tabular}
\caption{Impact of duplication on $\ppldedup$. Lower is better. The right column shows results for Languini's GPT-medium model with 370M parameters (vs 111M for small), trained on 2.8B tokens (vs 1.2B for small).}
\label{tab:dup_basic}
\vspace{-9pt}
\end{table}

\subsubsection{Empirical Implementation}

To achieve the perfect duplication described above, we simulate Eq.~\ref{eq:perfect_duplicate_choice} by duplicating every entry in our subword vocabulary.
First, we assume our BPE-generated vocabulary is composed of canonical symbols $\Sigmaduplicated = \{\dubwordone, \dubwordtwo, ...\}$.
This set is composed of 16k subwords.
We then duplicate each to get a vocabulary of size 32k:\footnote{Model $\model(\words)$ thus has more embedding parameters than $\model(\dubwords)$; these extra parameters, however, should not yield an unfair advantage (see \cref{app:embedding_params}).} $\alphabet=\{\wordone, \wordone', \wordtwo, \wordtwo',  ...\}$. 
This gives us the deduplication mapping $\alphabetdupsmap(\wordnumber{i}) = \alphabetdupsmap(\wordnumber{i}') = \dubwordnumber{i}$.
Given a sequence $\dubwords$ from our training set $\trainset$, we then create a duplicated sequence $\words$ by independently sampling the form of each token $\dubwordnumber{i} \in \dubwords$ to be either $\wordnumber{i}$ or $\wordnumber{i}'$.
Unless specified differently, we set all duplicate choice probabilities to be uniform (i.e., $\prdontduplicate=\prduplicate=0.5$ for all $i$).

\subsubsection{Raw Performance}

In this section, we describe our main results using the perfect duplication setting.
First, we note that, when training with $\prduplicate = 0.5$, 
each subword $\wordnumber{i}$ or $\wordnumber{i}'$ is only seen half as often by $\model(\words)$ as its original version $\dubwordnumber{i}$ is seen by $\model(\dubwords)$.
However, we achieve significantly better performance with $\model(\words)$ than with $\model(\dubwords)$ trained on only 50\% of the dataset (see \cref{tab:dup_basic}). 
The model thus seems to generalise across duplicates, with data containing $\word$ leading to improved performance on $\word'$ and vice versa.
Still, duplication significantly hurts performance.
Using the duplicated data, the model $\model(\words)$ is only about $85\%$ as data efficient as $\model(\dubwords)$, which is trained on the original data.
This suggests our LMs cannot generalise perfectly.
Interestingly, this trend stays relatively consistent across different amounts of training data (up to 3x, see \cref{fig:scaling_dup}) and also applies to a 3x larger GPT-medium model (see \cref{tab:dup_basic}). Further, if we vary the number of subwords we duplicate, interpolating from 0\% of the vocabulary ($\model(\dubwords)$ setting) to 100\% of the vocabulary ($\model(\words)$ setting), we obtain a roughly linear increase in $\ppldedup$ (see \cref{fig:frac_dup_ppl}).\looseness=-1

\begin{takeaway}
    The model is capable of generalising across duplicates, yet their presence negatively impacts performance.
\end{takeaway}

\begin{figure}[t]
    \centering
    \includegraphics[width=0.95\linewidth]{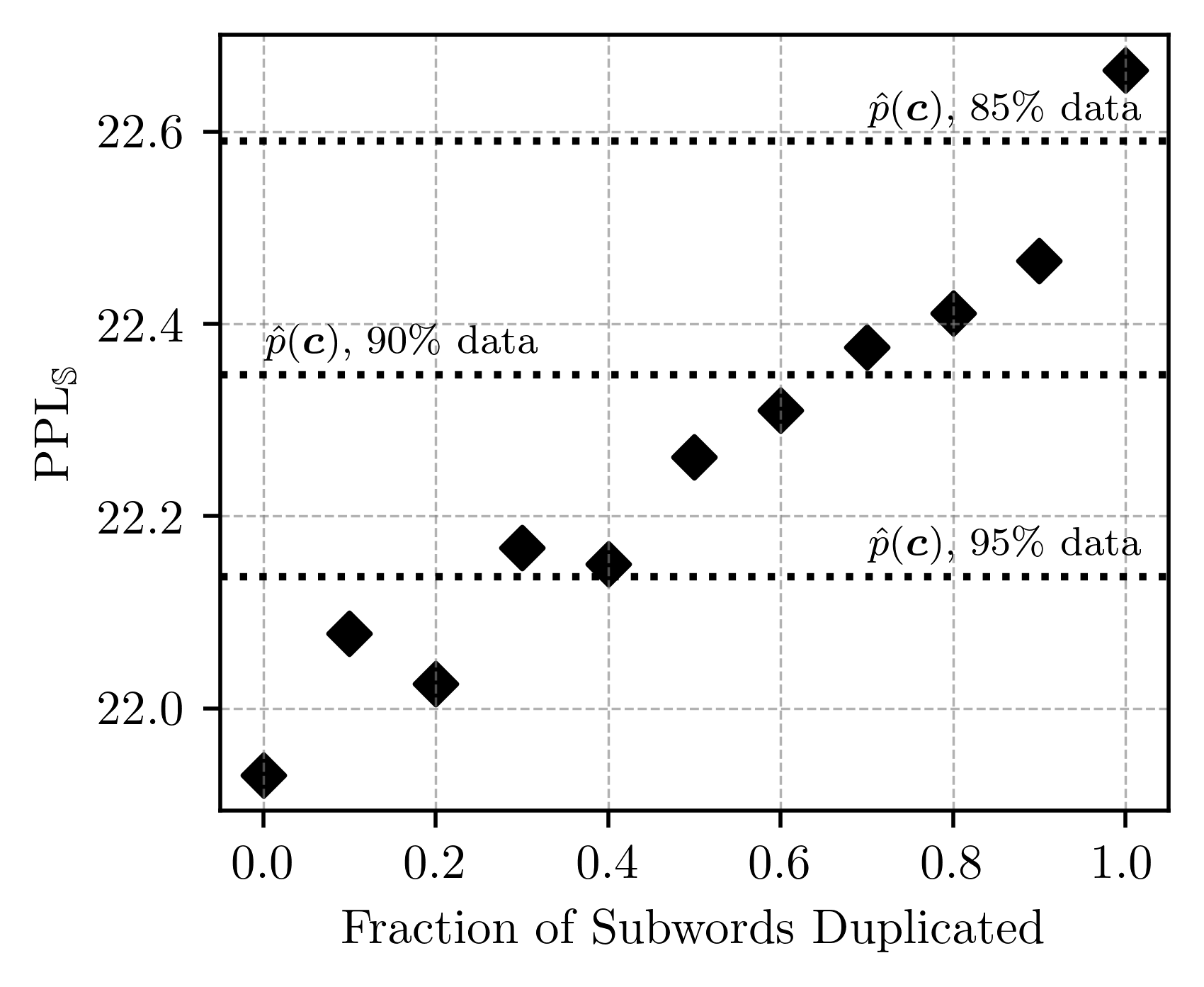}
    \vspace{-7pt}
    \caption{
    Impact of duplication on $\ppldedup$ while varying the fraction of subwords in the vocabulary that are duplicated ($1.0$ corresponds to $\model(\words)$, $0.0$ to $\model(\dubwords$)). Lower $\ppldedup$ is better. When duplicating 70\% of the vocabulary (which yields a 41\% duplication rate in the final vocabulary, roughly the rate of near duplicates in real vocabularies), we obtain $\ppldedup \approx 22.4$; this is equivalent to a $\approx 10\%$ decrease in data efficiency.
    }
    \label{fig:frac_dup_ppl}
    \vspace{-10pt}
\end{figure}

\subsubsection{Duplicates' Alignment}

One strategy the model could use to generalise across duplicated pairs $\wordnumber{i}$ and $\wordnumber{i}'$ is to fully align their representations.
If these word pairs have a cosine similarity of $1.0$, then any change to other model components would affect them similarly.
When we analyse our model $\model(\words)$'s  embeddings, we indeed observe a high average cosine similarity of around 0.8 among duplicate pairs. 
This number is even higher for frequent subwords (see \cref{fig:dup_cossim_freq});
it appears that a subword's frequency during training is an underlying driver for alignment.
This observation is intuitive: the representations of $\wordnumber{i}$ and $\wordnumber{i}'$ are both randomly initialised and converge to each other after a certain number of gradient updates.\looseness=-1

What causes this alignment of representations and what is it impacted by?
Loosely speaking, the contexts in which the duplicates appear follow the same distribution; this might lead to similar gradient signals throughout training.
If this is the case, then this high cosine similarity should not be an exclusive property of transformers, but apply to simpler architectures as well.
Interestingly, when training a word2vec model \cite{mikolov2013efficient} on the same data, we observe its embeddings exhibit even stronger alignment (details in \cref{app:emb_similarity}).

\begin{takeaway}
    \label{tak:representations_aligned}
    Representations of frequent duplicates have high cosine similarity.
\end{takeaway}

\begin{figure}[t]
    \centering
    \includegraphics[width=0.95\linewidth]{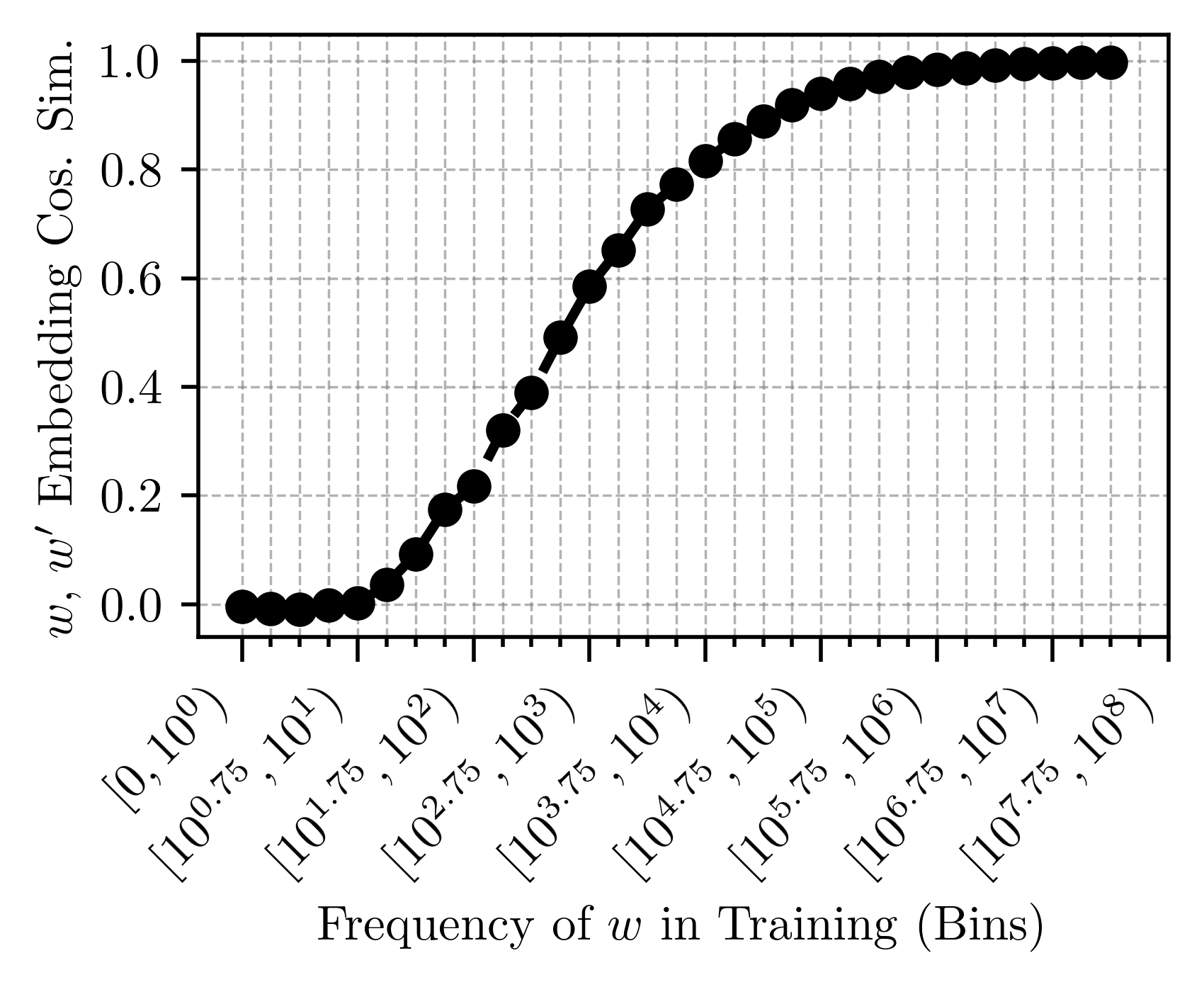}
    \vspace{-10pt}
    \caption{Input embedding cosine similarity of duplicates $\wordnumber{i}$, $\wordnumber{i}'$, by frequency. Frequencies binned and similarities averaged per bin.}
    \label{fig:dup_cossim_freq}
    \vspace{-9pt}
\end{figure}

\subsubsection{Finetuning Generalisation}

Presumably, the alignment of duplicated word pairs may cause the model to use the same ``circuits'' when processing those words, allowing it to generalise what it learns across them \cite{cammarata2020thread, elhage2021mathematical}.
To verify this hypothesis, we finetune our $\model(\words)$ model on GLUE \cite{wang2019glue}, employing only one subword from each duplicate pair as input (specifically, $\wordnumber{i}$ and never $\wordnumber{i}'$).
Intriguingly, when this finetuned model is subsequently evaluated on the unseen subwords (i.e., $\wordnumber{i}'$) it generalises perfectly, achieving the same accuracy (see \cref{tab:glue}).

\begin{table}[t]
\centering
\begin{tabular}{lcc}
\toprule
 & \multicolumn{2}{c}{\textbf{GLUE Accuracy}} \\ \cmidrule(lr){2-3}
\textbf{Model} & \textbf{on \(\wordnumber{i}\)} & \textbf{on \(\wordnumber{i}'\)} \\
\midrule
$\model(\dubwords)$ & 0.72 & - \\
$\model(\words)$ & 0.71 & 0.71 \\
\bottomrule
\end{tabular}
\caption{GLUE average validation accuracy achieved by fine-tuning solely on \(\wordnumber{i}\) inputs (and not on \(\wordnumber{i}'\)) while keeping the embedding layer frozen.}
\label{tab:glue}
\vspace{-5pt}
\end{table}

\begin{takeaway}
GLUE performance of $\model(\words)$ generalises across duplicate pairs despite being finetuned with $\wordnumber{i}$ and evaluated with $\wordnumber{i}'$.\looseness=-1
\end{takeaway}

\subsubsection{Comparing Duplicating or Not a Pair}

We observed that the alignment of duplicate subword representations seems to drive generalisation.
If infrequent subwords have less aligned representations, does this mean that they generalise less? 
To investigate the effects of duplication in a more controlled manner, we now duplicate only half of our vocabulary and train a model in this setting. 
This gives us a treatment group of duplicated subwords and a control group of non-duplicated subwords, allowing us to isolate the causal effect of duplication. 
We will now analyse this effect on the output side (predicting duplicated tokens compared to non-duplicated tokens) and on the input side (predictions based on a context with many duplicated tokens vs few duplicated tokens).

\begin{table}[t]
\centering
\begin{tabular}{lcc}
\toprule
\textbf{Subset} & \textbf{Mean $\mathbf{\Delta}$ Surprisal} \\
\midrule
All                             & 0.015 \\
Duplicated                      & 0.018  \\
Not Duplicated                  & 0.012  \\
\bottomrule
\end{tabular}
\vspace{-3pt}
\caption{Duplication of half the vocabulary: mean difference in surprisal to $\modeldub(\dubwords)$ for subwords within the treatment (duplicated) and control (non-duplicated) groups.}
\label{tab:dup_half_output_causal}
    \vspace{-9pt}
\end{table}

\begin{figure}[t]
    \centering
    \includegraphics[width=0.9\linewidth]{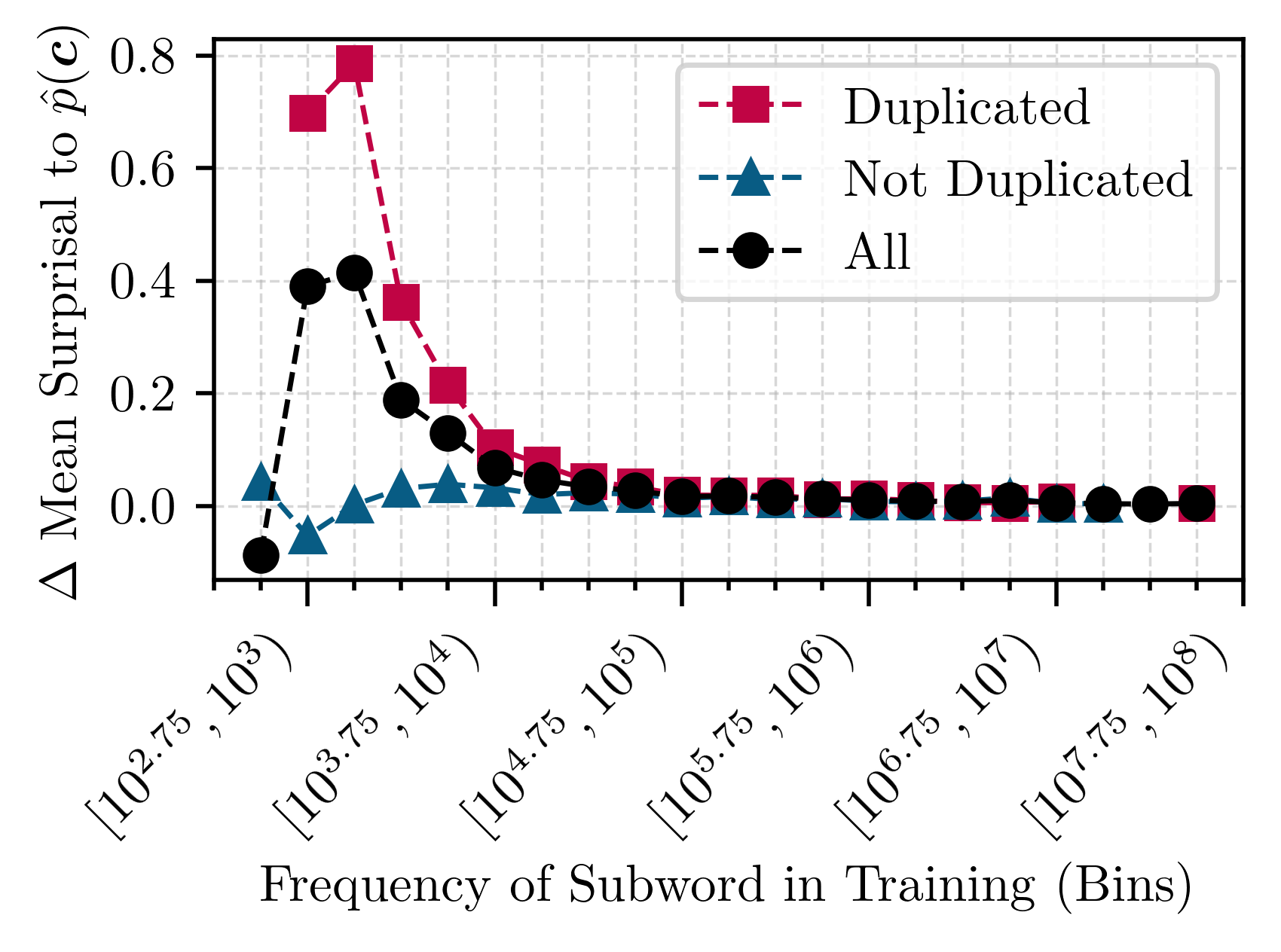}
    \vspace{-7pt}
    \caption{
         Duplication of half of the vocabulary: Analysing the mean surprisal difference per subword between $\model(\words)$ and $\model(\dubwords)$.
         Frequencies are categorised into bins, with averages computed for each bin.\protect\footnotemark
    }
    \label{fig:dup_output_surprisal_freq}
    \vspace{-9pt}
\end{figure}

\footnotetext{To reduce noise and ensure readability, we only consider subwords that occur at least 10 times in the test set and bins that contain at least 3 subwords.}

\newcommand{\countfunc}[1]{\mathrm{count}(#1)}
\newcommand{\concat}{\|}

\begin{figure*}[t]
    \centering
    \includegraphics[width=0.93\linewidth]{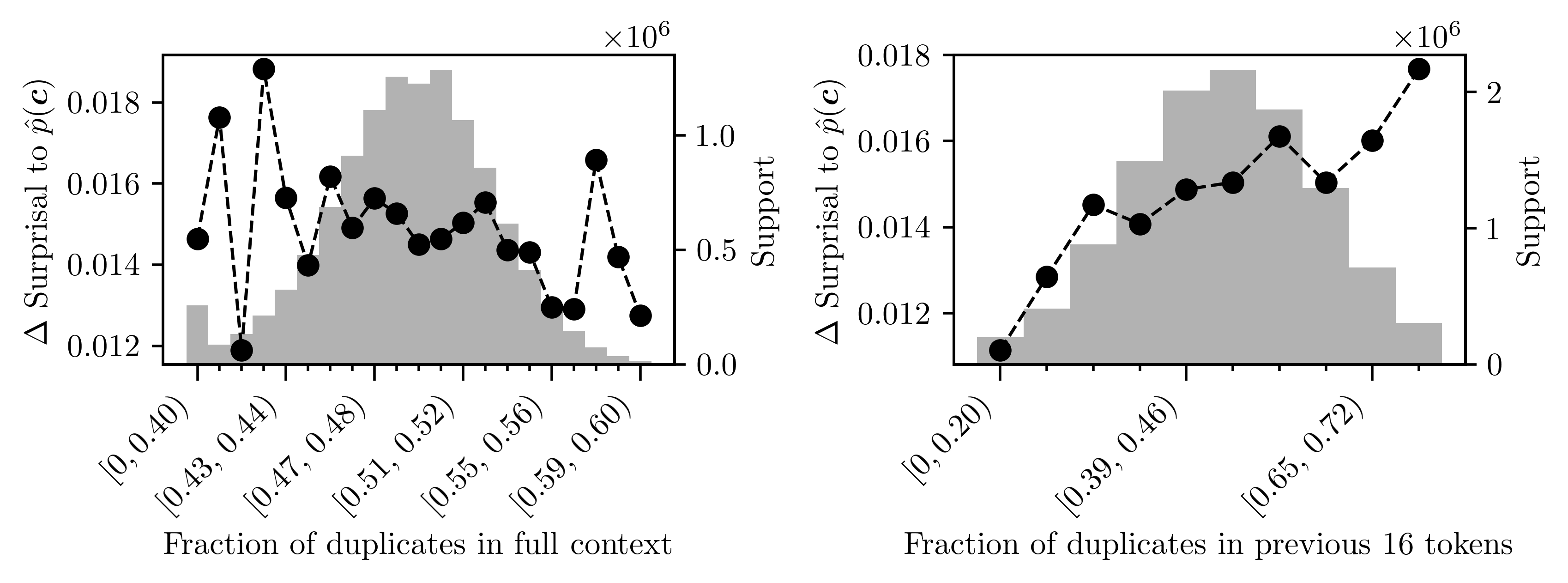}
    \vspace{-7pt}
    \caption{
    Duplication of half of the vocabulary. Difference between the surprisal assigned to each token by $\model(\words)$ and $\model(\dubwords)$, depending on fraction $\frac{\text{duplicated subwords}}{\text{non-duplicated subwords}}$ in context. 
    Fractions are categorised into bins, with average surprisal differences computed for each bin. ``Support'' shows the number of samples per bin.
    }
    \label{fig:dup_input_surprisal}
    \vspace{-10pt}
\end{figure*}

On the output side, we measure the effect of duplication by first taking the difference in surprisal (negative log probability) assigned to each original subword token $\dubword_t$ by either model $\model(\words)$ (through $\modelproj(\dubword_t \mid \words_{<t})$) or $\model(\dubwords)$ (through $\model(\dubword_t \mid \dubwords_{<t})$).
We then average these delta surprisals within either the set of actually duplicated (treatment) subwords, or the ones not duplicated (control).
We observe that subword duplication leads to an increase in surprisal which is around 50\% higher when predicting duplicated tokens compared to non-duplicated ones (see \cref{tab:dup_half_output_causal}).
This might seem unexpected, considering that, e.g., a simple n-gram model's predictions would not be affected by duplicated outputs.\footnote{An $n$-gram's probabilities are defined by count statistics, and:
$\frac{\countfunc{\dubword \concat \text{context}}}{\countfunc{\text{context}}} = \frac{\countfunc{\word \concat \text{context}}}{\countfunc{\text{context}}} + \frac{\countfunc{\word' \concat \text{context}}}{\countfunc{\text{context}}}$}
Yet, in our LM, the output embeddings of duplicated subwords receive only half as many gradient updates, which could explain the lower performance.
In line with this, we observe that infrequent subwords are affected the most by this performance loss (see \cref{fig:dup_output_surprisal_freq}), likely because their output representations are not well aligned after receiving few updates.\looseness=-1

To evaluate the impact of duplication on the LM's input side, we assess how surprisal on subwords changes depending on the number of duplicated (vs non-duplicated) subwords in its context.
Here, we observe no clear trend between the number of duplicates in our model's full context window and its performance loss on observed subwords (see \cref{fig:dup_input_surprisal}).
However, when limited to a more local context of size 16, the number of duplicate tokens seems to have a clear negative impact on prediction performance.\looseness=-1

\begin{takeaway}
Infrequent subwords are predicted worse when duplicated, and duplicated tokens in the LM's local context tend to hurt its predictions.
\end{takeaway}

\vspace{-3pt}
\subsection{Natural Duplication} \label{sec:results_natural_duplication}

After exploring the impact of perfect duplicates in LMs, we now turn our attention to the influence of naturally occurring near duplicates on LM performance.
Notably, despite their similarity, near duplicates are seldom perfectly interchangeable.
For example, \subword{\_individual} differs from \subword{\_individuals}, the \subword{\_he} in `and he writes' does not convey the same meaning as \subword{he} in `breathe' (we analyse this in \cref{app:space_near_dup}),
and a \subword{Now} at the start of a sentence can subtly vary from a \subword{now} used mid-sentence.
When merging near duplicates, we lose information about such small differences. If this lost information is relevant for predicting a token $\dubword_t = \alphabetdupsmap(\word_t)$, the task should become harder through deduplication.
We have $\mi(\Words_{<\T}; \Dubwords_{\T} \mid \Dubwords_{< \T}) > 0$ and hence
\begin{equation}
\ent(\Dubwords) > \ent_{\alphabetdupsmap}(\Words)
\end{equation}
which means that even a perfect model would perform worse on the deduplicated data distribution.
This subsection examines whether merging naturally occurring near duplicates yields comparable advantages to merging perfect duplicates.

\newcommand{\extraemb}{e_\text{non-canonical}}

\begin{table*}
\centering
\begin{tabular}{lcccc}
\toprule
      & \multicolumn{4}{c}{$\ppldedupbf$} \\ \cmidrule(lr){2-5}
\textbf{Setting} & $\whitespacedupsmap$ & $\lowerdupsmap$ & $\pluraldupsmap$ & $\alphabetdupsmap_\text{all}$\\
\midrule
$\model(\words)$$^*$                                     & 21.80 (± 0.11)    & 21.50 (± 0.11)  &  21.56 (± 0.11)  & 20.99 (± 0.11)    \\  

$\model(\words)$, 95\% data                       & 22.00     & 21.71  &  21.77   &  21.56   \\ 
$\model(\words)$, 90\% data                       & 22.21     & 21.91  &  21.97  &   21.76  \\ 
\midrule
$\model(\dubwords)$$^*$   & 22.14 (± 0.16)     & 21.87 (± 0.15) &  21.82 (± 0.11)  & 22.08 (± 0.09)\\ 
\midrule
\makecell[l]{$\model(\dubwords) + \extraemb$$^*$ }  & 21.92 (± 0.17) & 21.45 (± 0.12)  &  21.59 (± 0.14) & 21.15 (± 0.09)\\ 
\bottomrule
\end{tabular}
\caption{Impact of deduplication on $\ppldedup$. For rows with $^*$, the reported values represent means and standard deviations over four runs. Note that, for each column, $\model(\dubwords)$ refers to a different model trained under the respective $\alphabetdupsmap$.\looseness=-1}
\label{tab:dedup_basic}
\vspace{-10pt}
\end{table*}

\subsubsection{Empirical Implementation}
\label{sec:deduplication_implementation}

We consider four types of near duplicates, defined by their corresponding mapping $\alphabetdupsmap$:
\begin{description}\itemsep0em 
    \item[$\whitespacedupsmap$] ignores leading or trailing whitespace, mapping, e.g., \subword{\_the} to \subword{the}. Around 10\% of subwords are merged.
    \item[$\lowerdupsmap$] ignores capitalisation, mapping, e.g., \subword{Now} to \subword{now}. Around 12\% of subwords are merged.\looseness=-1
    \item[$\pluraldupsmap$] ignores plural markings,\footnote{For simplicity, we implement it as an easy rule-based mapping: If a subword has a trailing \subword{s} and at least four characters, not counting whitespace, we ignore the \subword{s}. While this approach yields false positives, manual inspections suggests the results are generally acceptable.} mapping, e.g., \subword{\_individuals} to \subword{\_individual}. Around 8\% of subwords are merged.
    \item[$\alldupsmap$] combines all of the previous mappings, mapping, e.g., \subword{\_Books} to \subword{book}. Around 29\% of subwords are merged.\looseness=-1
\end{description}
We train models $\modeldub(\dubwords)$ for each of these deduplication mappings $\alphabetdupsmap$ and compare them to the performance of the regular model $\model(\words)$ when evaluated under the same projection.

The first three mappings all reduce the vocabulary size by around 10\%. 
If the near duplicates were perfectly equivalent, we could expect $\modeldub(\dubwords)$ to perform marginally better than $\model(\words)$, which is what we observe in the synthetic setting with a low duplication rate (see \cref{fig:frac_dup_ppl}).
For the combined $\alldupsmap$, with 29\% near duplicates, we would expect a performance boost corresponding to roughly 5\% better data efficiency; or, alternatively, a $\ppldedup$ decrease of roughly 0.2 from $\model(\words)$ to $\modeldub(\dubwords)$ (see the results for 40\% duplication rate\footnote{A ratio of 29\% of duplicates in vocabulary $\alphabet$ corresponds to duplicating around 40\% of the initial $\Sigmaduplicated$ ($\frac{0.4}{1 + 0.4} \approx {0.29}$).} in \cref{fig:frac_dup_ppl}).

Modern large models show even higher duplication rates than our small vocabulary. 
Their vocabularies contain around 40\% near duplicates under $\alldupsmap$ (see \cref{tab:llm_dup_stats_short}). 
This translates to duplicating roughly 70\% of a canonical vocabulary; in the synthetic case, this corresponds to a potential
data efficiency increase of around 10\% due to deduplication, assuming the trends in \cref{fig:frac_dup_ppl} apply.
In the following, we examine whether this is plausible by verifying to what extent these trends transfer from perfect duplicates to near duplicates for our models.\looseness=-1

Note that, different to the previous subsection where we duplicated our standard vocabulary, here we deduplicate it; this means the $\model(\dubwords)$ in the previous subsection and $\model(\words)$ in this subsection both refer to a ``baseline" over our standard vocabulary.

\subsubsection{Raw Performance \& Duplicate Alignment}

We observe that deduplication hurts performance (see \cref{tab:dedup_basic}).
Across all considered $\alphabetdupsmap$, training $\model(\dubwords)$ on deduplicated data yields worse results than the setting $\model(\words)$ in which near duplicates remain untouched.
The performance degradation is equivalent to training on 5-10\% less data.
This is in stark contrast to the trends we observed on perfectly equivalent synthetic duplicates, where deduplication boosts performance. 
Near duplicates thus seem to be less equivalent than one might expect, noticeably differing in their semantics.\looseness=-1

The near duplicates' learned embeddings, which tend to contain information about the subwords' semantics, reflect this discrepancy.
Near duplicates’ embeddings are not nearly as similar as the embeddings of synthetic perfect duplicates, which had a cosine similarity of 0.8: all three types of natural duplicates exhibit average cosine similarities of around 0.4.
This indicates that the model perceives semantic differences between the near duplicates. 
As described earlier, if we merge them, we will lose information. 
The associated decline in performance indicates that the information is significant.

\begin{takeaway}
Near duplicates are not equivalent and merging them hurts performance.
\end{takeaway}

\begin{table}[t]
\centering
\begin{tabular}{lcc}
\hline
 & \multicolumn{2}{c}{\textbf{Mean $\mathbf{\Delta}$ Surprisal}} \\\cmidrule(lr){2-3}
\textbf{Setting} & Deduplicated  & Not Deduplicated \\
\hline
$\whitespacedupsmap$       & 0.014  & 0.010 \\
$\lowerdupsmap$            & 0.009  & 0.008 \\
$\pluraldupsmap$           & 0.010  & 0.009\\
\hline
\end{tabular}
\vspace{-5pt}
\caption{
Deduplication of half of the vocabulary. Mean difference to baseline surprisal, for subwords in treatment (duplicated) and control (not duplicated) group.
}
\label{tab:dedup_half_output_causal}
\vspace{-10pt}
\end{table}

\subsubsection{Comparing Deduplicating or Not a Pair}

Are predictions worse when more deduplicated tokens are in the context, since more information is ``lost''?
We investigate this again through a controlled experiment. 
We run experiments where we deduplicate only half of the near duplicates, obtaining a deduplicated treatment group and an unchanged control group.
This setup is similar to the one in the previous subsection, but treatment and control group do not cover the entire vocabulary; both consist only of subwords that have a near duplicate. 
To isolate the causal effect of duplication, we investigate differences between these groups.

We first compare the effects of the number of deduplicated subwords and non-deduplicated subwords in the LMs context, i.e., measuring the effect of deduplication on the input side. 
An increased number of deduplicated subwords seems to generally lead to a steeper increase in surprisal than an increased number of non-deduplicated subwords, at least in the local context (see \cref{fig:dedup_input_causal_surprisal} in \cref{app:dedup_subwords_in_context}).
This suggests that deduplicated subwords in the context do hurt performance, presumably due to the lost information.\looseness=-1

Interestingly, deduplicated subwords are also predicted slightly worse on the output side (see \cref{tab:dedup_half_output_causal}).
In this setting, the model is forced to learn only one output embedding for each pair of merged near duplicates. 
If these merged subwords have different meanings, finding a single embedding to represent both of them might be challenging, as this embedding would need to produce large dot products with the hidden states of contexts from both subwords.\looseness=-1

\begin{takeaway}
    The presence of merged near duplicates in the local context tends to reduce an LM's prediction accuracy.
\end{takeaway}

\subsubsection{Re-adding Information}

If the missing information in the input causes worse predictions, can we improve performance by re-injecting it? If we provide the model with an extra input that allows it to distinguish between merged near duplicates while still sharing their embeddings, could we obtain the benefits of deduplication without the downside of losing information?

To investigate this, we introduce an extra learnable embedding $\extraemb$.
This is a single shared vector which is added to the model's input at every token that corresponds to a non-canonical subword.
(A non-canonical subword $\word$ is any subword where $\alphabetdupsmap(\word) \neq \word$, e.g., \subword{Now} when operating under $\lowerdupsmap$, since $\lowerdupsmap(\text{\subword{Now}}) = \text{\subword{now}} \neq \text{\subword{Now}}$.)
As $\alldupsmap$ combines the three different types of deduplication, we use three different learnable embeddings for the $\alldupsmap$ setting, one for each type.

When comparing $\model(\dubwords)$ with and without $\extraemb$, such an embedding's availability significantly improves performance ($\mathrm{p} \leq 0.05$ in one-sided t-test) in all four deduplication settings (see \cref{tab:dedup_basic}).
Further, $\modeldub(\dubwords)$ with $\extraemb$ roughly matches $\model(\words)$ performance for $\whitespacedupsmap$, $\lowerdupsmap$, and $\pluraldupsmap$ (differences not statistically signficant with $\mathrm{p} = 0.22, 0.48, 0.72$, respectively, in two-sided t-test).
However, we still do not see the same performance improvements due to deduplication as we observed in the synthetic setting; this is especially clear for $\alldupsmap$ which has a higher deduplication rate, but clearly fails to improve performance ($\modeldub(\dubwords)$ with $\extraemb$ is significantly worse than $\model(\words)$ with $\mathrm{p} \leq 0.05$ in two-sided t-test).
Presumably, even if they belong to the same type ($\whitespacedupsmap$, $\lowerdupsmap$, or $\pluraldupsmap$), near duplicate pairs' semantic differences are diverse and can thus not be fully captured by a single shared embedding vector.\looseness=-1

\begin{takeaway}
    Performance losses can be mitigated by accounting for semantic differences of near duplicates via a shared learned embedding.
    Still, this approach does not achieve the same benefits as observed when merging perfectly equivalent duplicates.\looseness=-1
\end{takeaway}

\section{Related Work}

Our experiments are partly inspired by a number of works on cross-lingual LM generalisation.
These works also use duplicated vocabularies, terming the duplicates ``fake-english''.
Unlike our work, however, they sample entire sequences of either English or ``fake-english'' tokens \citep[we perform token-wise i.i.d. sampling;][]{wang2019cross, dufter2020identifying, schafer2024language}. Relatedly, \citeposs{huang2023lexinvariant} lex-invariant LMs can be interpreted as an extreme case of token duplication; essentially creating infinite ``fake languages'', they show that LMs can learn, to a certain extent, even in the absence of a fixed set of embeddings.
More similar to our duplication method is the work of \citet{kharitonov2021bpe}, who use duplicates sampled at the token level to isolate the effect of vocabulary size when investigating the role of BPE tokenisation in a model's ability to memorise training data.\looseness=-1

Another set of related work studies subword regularisation techniques, such as, e.g.,  BPE dropout \cite{kudo-2018-subword, provilkov-etal-2020-bpe}.
Certain character sequences can be represented by multiple equivalent sequences of subwords.
These techniques then non-deterministically choose between these equivalent choices at training time, which can lead to improved LMs.
Our perfect duplication experiments, where we train LMs with either of two duplicated tokens, can be seen as similar to these methods.
Unlike them, however, our approach duplicates a model's vocabulary to introduce this ambiguity, which may explain the fact that our results diverge.
Relatedly, our projected entropy measure (which sums over subword duplicates) is analogous to the marginalisation proposed by \citet{cao-rimell-2021-evaluate}, which sums over spurious tokenisations.

\section{Conclusion}

In this paper, we investigate to what extent (near) duplicate subwords impact language modelling performance, conducting controlled experiments on synthetic perfect duplicates and natural near duplicates. 
We find that LMs can generalise across duplicated subwords, although this incurs extra cost. When operating on a fully duplicated vocabulary, the LM is about 17\% less data efficient.
This number depends roughly linearly on the fraction of the vocabulary that is duplicated.
Assuming that roughly 40\% of subwords are near duplicates in common LM vocabularies, our findings imply that LMs with improved generalisation across duplicates, e.g., by modelling language at the character-level, could achieve data efficiency gains of up to 10\%. 
This bound is reached for perfectly equivalent duplicates. 
However, we find that natural near duplicates are \emph{not} perfectly equivalent: in practice, the potential for such performance improvements is likely limited.

\section*{Limitations}

We conduct most of our experiments on models with about 100M parameters, training on roughly 1.2B tokens in English.
Although the trends we identify are consistent in up to 3x larger models and datasets, it is uncertain whether they extend to the scale of modern large language models.\footnote{Relatedly, \citet{bandel-etal-2022-lexical} find that larger models trained over extended periods tend to exhibit decreased dependency on lexical overlap.}
Furthermore, we have not validated that our results transfer to languages beyond English.

When we study the causal effects of (de)duplication on the input and output side, we do not fully isolate the two effects from each other.
After a manual inspection, it appears that this does not confound the results; e.g., duplicated subwords are not more frequent in the input/context of duplicated subwords than non-duplicated subwords are.
However, to fully exclude the possibility that such patterns affect results, one should run experiments where only the inputs or only the outputs are (de)duplicated.

Finally, while studying natural near duplicates, we only investigate deduplication mappings that can be described by simple heuristics. 
Alternatively, one could use more involved methods to also deduplicate, e.g., typos or wordform variations beyond plural forms.

\section*{Acknowledgements}

The authors would like to thank Shauli Ravfogel, Andreas Opedal, Marius Mosbach, and the anonymous reviewers for feedback on earlier versions of this manuscript.

\bibliography{custom}

\begin{thebibliography}{35}
\expandafter\ifx\csname natexlab\endcsname\relax\def\natexlab#1{#1}\fi

\bibitem[{Anthropic(2023)}]{claude2}
Anthropic. 2023.
\newblock \href
  {https://www-cdn.anthropic.com/bd2a28d2535bfb0494cc8e2a3bf135d2e7523226/Model-Card-Claude-2.pdf}
  {Model card and evaluations for claude models}.
\newblock \emph{Anthropic Blog}.

\bibitem[{Bandel et~al.(2022)Bandel, Goldberg, and
  Elazar}]{bandel-etal-2022-lexical}
Elron Bandel, Yoav Goldberg, and Yanai Elazar. 2022.
\newblock \href {https://doi.org/10.18653/v1/2022.findings-emnlp.323} {Lexical
  generalization improves with larger models and longer training}.
\newblock In \emph{Findings of the Association for Computational Linguistics:
  EMNLP 2022}, pages 4398--4410, Abu Dhabi, United Arab Emirates. Association
  for Computational Linguistics.

\bibitem[{Brown et~al.(2020)Brown, Mann, Ryder, Subbiah, Kaplan, Dhariwal,
  Neelakantan, Shyam, Sastry, Askell, Agarwal, Herbert-Voss, Krueger, Henighan,
  Child, Ramesh, Ziegler, Wu, Winter, Hesse, Chen, Sigler, Litwin, Gray, Chess,
  Clark, Berner, McCandlish, Radford, Sutskever, and
  Amodei}]{brown2020language}
Tom Brown, Benjamin Mann, Nick Ryder, Melanie Subbiah, Jared~D Kaplan, Prafulla
  Dhariwal, Arvind Neelakantan, Pranav Shyam, Girish Sastry, Amanda Askell,
  Sandhini Agarwal, Ariel Herbert-Voss, Gretchen Krueger, Tom Henighan, Rewon
  Child, Aditya Ramesh, Daniel Ziegler, Jeffrey Wu, Clemens Winter, Chris
  Hesse, Mark Chen, Eric Sigler, Mateusz Litwin, Scott Gray, Benjamin Chess,
  Jack Clark, Christopher Berner, Sam McCandlish, Alec Radford, Ilya Sutskever,
  and Dario Amodei. 2020.
\newblock \href
  {https://proceedings.neurips.cc/paper_files/paper/2020/file/1457c0d6bfcb4967418bfb8ac142f64a-Paper.pdf}
  {Language models are few-shot learners}.
\newblock In \emph{Advances in Neural Information Processing Systems},
  volume~33, pages 1877--1901. Curran Associates, Inc.

\bibitem[{Cammarata et~al.(2020)Cammarata, Carter, Goh, Olah, Petrov, Schubert,
  Voss, Egan, and Lim}]{cammarata2020thread}
Nick Cammarata, Shan Carter, Gabriel Goh, Chris Olah, Michael Petrov, Ludwig
  Schubert, Chelsea Voss, Ben Egan, and Swee~Kiat Lim. 2020.
\newblock \href {https://doi.org/10.23915/distill.00024} {Thread: Circuits}.
\newblock \emph{Distill}.

\bibitem[{Cao and Rimell(2021)}]{cao-rimell-2021-evaluate}
Kris Cao and Laura Rimell. 2021.
\newblock \href {https://doi.org/10.18653/v1/2021.emnlp-main.161} {You should
  evaluate your language model on marginal likelihood over tokenisations}.
\newblock In \emph{Proceedings of the 2021 Conference on Empirical Methods in
  Natural Language Processing}, pages 2104--2114, Online and Punta Cana,
  Dominican Republic. Association for Computational Linguistics.

\bibitem[{Clark et~al.(2022)Clark, Garrette, Turc, and
  Wieting}]{clark-etal-2022-canine}
Jonathan~H. Clark, Dan Garrette, Iulia Turc, and John Wieting. 2022.
\newblock \href {https://doi.org/10.1162/tacl_a_00448} {{CANINE}: Pre-training
  an efficient tokenization-free encoder for language representation}.
\newblock \emph{Transactions of the Association for Computational Linguistics},
  10:73--91.

\bibitem[{Dufter and Sch{\"u}tze(2020)}]{dufter2020identifying}
Philipp Dufter and Hinrich Sch{\"u}tze. 2020.
\newblock \href {https://aclanthology.org/2020.emnlp-main.358.pdf} {Identifying
  elements essential for {BERT}’s multilinguality}.
\newblock In \emph{Proceedings of the 2020 Conference on Empirical Methods in
  Natural Language Processing (EMNLP)}, pages 4423--4437.

\bibitem[{Elhage et~al.(2021)Elhage, Nanda, Olsson, Henighan, Joseph, Mann,
  Askell, Bai, Chen, Conerly, DasSarma, Drain, Ganguli, Hatfield-Dodds,
  Hernandez, Jones, Kernion, Lovitt, Ndousse, Amodei, Brown, Clark, Kaplan,
  McCandlish, and Olah}]{elhage2021mathematical}
Nelson Elhage, Neel Nanda, Catherine Olsson, Tom Henighan, Nicholas Joseph, Ben
  Mann, Amanda Askell, Yuntao Bai, Anna Chen, Tom Conerly, Nova DasSarma, Dawn
  Drain, Deep Ganguli, Zac Hatfield-Dodds, Danny Hernandez, Andy Jones, Jackson
  Kernion, Liane Lovitt, Kamal Ndousse, Dario Amodei, Tom Brown, Jack Clark,
  Jared Kaplan, Sam McCandlish, and Chris Olah. 2021.
\newblock \href {https://transformer-circuits.pub/2021/framework/index.html} {A
  mathematical framework for transformer circuits}.
\newblock \emph{Transformer Circuits Thread}.

\bibitem[{Gage(1994)}]{gage1994new}
Philip Gage. 1994.
\newblock A new algorithm for data compression.
\newblock \emph{C Users Journal}, 12(2):23--38.

\bibitem[{Gao et~al.(2020)Gao, Biderman, Black, Golding, Hoppe, Foster, Phang,
  He, Thite, Nabeshima, Presser, and Leahy}]{gao2020pile}
Leo Gao, Stella Biderman, Sid Black, Laurence Golding, Travis Hoppe, Charles
  Foster, Jason Phang, Horace He, Anish Thite, Noa Nabeshima, Shawn Presser,
  and Connor Leahy. 2020.
\newblock \href {https://arxiv.org/pdf/2101.00027.pdf} {The {Pile}: An 800{GB}
  dataset of diverse text for language modeling}.
\newblock \emph{arXiv preprint arXiv:2101.00027}.

\bibitem[{{Gemma Team} et~al.(2024){Gemma Team}, Mesnard, Hardin, Dadashi,
  Bhupatiraju, Pathak, Sifre, Rivière, Kale, Love, Tafti, Hussenot, Chowdhery,
  Roberts, Barua, Botev, Castro-Ros, Slone, Héliou, Tacchetti, Bulanova,
  Paterson, Tsai, Shahriari, Lan, Choquette-Choo, Crepy, Cer, Ippolito, Reid,
  Buchatskaya, Ni, Noland, Yan, Tucker, Muraru, Rozhdestvenskiy, Michalewski,
  Tenney, Grishchenko, Austin, Keeling, Labanowski, Lespiau, Stanway, Brennan,
  Chen, Ferret, Chiu, Mao-Jones, Lee, Yu, Millican, Sjoesund, Lee, Dixon, Reid,
  Mikuła, Wirth, Sharman, Chinaev, Thain, Bachem, Chang, Wahltinez, Bailey,
  Michel, Yotov, Sessa, Chaabouni, Comanescu, Jana, Anil, McIlroy, Liu,
  Mullins, Smith, Borgeaud, Girgin, Douglas, Pandya, Shakeri, De, Klimenko,
  Hennigan, Feinberg, Stokowiec, hui Chen, Ahmed, Gong, Warkentin, Peran,
  Giang, Farabet, Vinyals, Dean, Kavukcuoglu, Hassabis, Ghahramani, Eck,
  Barral, Pereira, Collins, Joulin, Fiedel, Senter, Andreev, and
  Kenealy}]{gemma}
{Gemma Team}, Thomas Mesnard, Cassidy Hardin, Robert Dadashi, Surya
  Bhupatiraju, Shreya Pathak, Laurent Sifre, Morgane Rivière, Mihir~Sanjay
  Kale, Juliette Love, Pouya Tafti, Léonard Hussenot, Aakanksha Chowdhery,
  Adam Roberts, Aditya Barua, Alex Botev, Alex Castro-Ros, Ambrose Slone,
  Amélie Héliou, Andrea Tacchetti, Anna Bulanova, Antonia Paterson, Beth
  Tsai, Bobak Shahriari, Charline~Le Lan, Christopher~A. Choquette-Choo,
  Clément Crepy, Daniel Cer, Daphne Ippolito, David Reid, Elena Buchatskaya,
  Eric Ni, Eric Noland, Geng Yan, George Tucker, George-Christian Muraru,
  Grigory Rozhdestvenskiy, Henryk Michalewski, Ian Tenney, Ivan Grishchenko,
  Jacob Austin, James Keeling, Jane Labanowski, Jean-Baptiste Lespiau, Jeff
  Stanway, Jenny Brennan, Jeremy Chen, Johan Ferret, Justin Chiu, Justin
  Mao-Jones, Katherine Lee, Kathy Yu, Katie Millican, Lars~Lowe Sjoesund, Lisa
  Lee, Lucas Dixon, Machel Reid, Maciej Mikuła, Mateo Wirth, Michael Sharman,
  Nikolai Chinaev, Nithum Thain, Olivier Bachem, Oscar Chang, Oscar Wahltinez,
  Paige Bailey, Paul Michel, Petko Yotov, Pier~Giuseppe Sessa, Rahma Chaabouni,
  Ramona Comanescu, Reena Jana, Rohan Anil, Ross McIlroy, Ruibo Liu, Ryan
  Mullins, Samuel~L Smith, Sebastian Borgeaud, Sertan Girgin, Sholto Douglas,
  Shree Pandya, Siamak Shakeri, Soham De, Ted Klimenko, Tom Hennigan, Vlad
  Feinberg, Wojciech Stokowiec, Yu~hui Chen, Zafarali Ahmed, Zhitao Gong, Tris
  Warkentin, Ludovic Peran, Minh Giang, Clément Farabet, Oriol Vinyals, Jeff
  Dean, Koray Kavukcuoglu, Demis Hassabis, Zoubin Ghahramani, Douglas Eck,
  Joelle Barral, Fernando Pereira, Eli Collins, Armand Joulin, Noah Fiedel,
  Evan Senter, Alek Andreev, and Kathleen Kenealy. 2024.
\newblock \href
  {https://storage.googleapis.com/deepmind-media/gemma/gemma-report.pdf}
  {Gemma: Open models based on gemini research and technology}.
\newblock \emph{Google Blog}.

\bibitem[{Huang et~al.(2023)Huang, Zelikman, Chen, Wu, Valiant, and
  Liang}]{huang2023lexinvariant}
Qian Huang, Eric Zelikman, Sarah~Li Chen, Yuhuai Wu, Gregory Valiant, and Percy
  Liang. 2023.
\newblock \href {https://openreview.net/forum?id=NiQTy0NW1L} {Lexinvariant
  language models}.
\newblock In \emph{Thirty-seventh Conference on Neural Information Processing
  Systems}.

\bibitem[{Jiang et~al.(2023)Jiang, Sablayrolles, Mensch, Bamford, Chaplot,
  de~las Casas, Bressand, Lengyel, Lample, Saulnier, Lavaud, Lachaux, Stock,
  Scao, Lavril, Wang, Lacroix, and Sayed}]{jiang2023mistral}
Albert~Q. Jiang, Alexandre Sablayrolles, Arthur Mensch, Chris Bamford,
  Devendra~Singh Chaplot, Diego de~las Casas, Florian Bressand, Gianna Lengyel,
  Guillaume Lample, Lucile Saulnier, Lélio~Renard Lavaud, Marie-Anne Lachaux,
  Pierre Stock, Teven~Le Scao, Thibaut Lavril, Thomas Wang, Timothée Lacroix,
  and William~El Sayed. 2023.
\newblock \href {https://arxiv.org/abs/2310.06825} {Mistral 7{B}}.
\newblock \emph{arXiv preprint arXiv:2310.06825}.

\bibitem[{Jiang et~al.(2024)Jiang, Sablayrolles, Roux, Mensch, Savary, Bamford,
  Chaplot, de~las Casas, Hanna, Bressand, Lengyel, Bour, Lample, Lavaud,
  Saulnier, Lachaux, Stock, Subramanian, Yang, Antoniak, Scao, Gervet, Lavril,
  Wang, Lacroix, and Sayed}]{jiang2024mixtral}
Albert~Q. Jiang, Alexandre Sablayrolles, Antoine Roux, Arthur Mensch, Blanche
  Savary, Chris Bamford, Devendra~Singh Chaplot, Diego de~las Casas, Emma~Bou
  Hanna, Florian Bressand, Gianna Lengyel, Guillaume Bour, Guillaume Lample,
  Lélio~Renard Lavaud, Lucile Saulnier, Marie-Anne Lachaux, Pierre Stock,
  Sandeep Subramanian, Sophia Yang, Szymon Antoniak, Teven~Le Scao, Théophile
  Gervet, Thibaut Lavril, Thomas Wang, Timothée Lacroix, and William~El Sayed.
  2024.
\newblock \href {http://arxiv.org/abs/2401.04088} {Mixtral of experts}.
\newblock \emph{arXiv preprint arXiv:2401.04088}.

\bibitem[{K et~al.(2020)K, Wang, Mayhew, and Roth}]{wang2019cross}
Karthikeyan K, Zihan Wang, Stephen Mayhew, and Dan Roth. 2020.
\newblock \href {https://openreview.net/forum?id=HJeT3yrtDr} {Cross-lingual
  ability of multilingual {BERT}: An empirical study}.
\newblock In \emph{International Conference on Learning Representations}.

\bibitem[{Kharitonov et~al.(2021)Kharitonov, Baroni, and
  Hupkes}]{kharitonov2021bpe}
Eugene Kharitonov, Marco Baroni, and Dieuwke Hupkes. 2021.
\newblock \href {https://arxiv.org/abs/2110.02782} {How {BPE} affects
  memorization in transformers}.
\newblock \emph{arXiv preprint arXiv:2110.02782}.

\bibitem[{Kim et~al.(2016)Kim, Jernite, Sontag, and Rush}]{kim2016character}
Yoon Kim, Yacine Jernite, David Sontag, and Alexander Rush. 2016.
\newblock \href {https://arxiv.org/pdf/1508.06615.pdf} {Character-aware neural
  language models}.
\newblock In \emph{Proceedings of the AAAI Conference on Artificial
  Intelligence}, volume~30.

\bibitem[{Kingma and Ba(2015)}]{kingma2015adam}
Diederik Kingma and Jimmy Ba. 2015.
\newblock \href {https://arxiv.org/pdf/1412.6980.pdf} {Adam: A method for
  stochastic optimization}.
\newblock In \emph{International Conference on Learning Representations}, San
  Diega, CA, USA.

\bibitem[{Kudo(2018)}]{kudo-2018-subword}
Taku Kudo. 2018.
\newblock \href {https://doi.org/10.18653/v1/P18-1007} {Subword regularization:
  Improving neural network translation models with multiple subword
  candidates}.
\newblock In \emph{Proceedings of the 56th Annual Meeting of the Association
  for Computational Linguistics (Volume 1: Long Papers)}, pages 66--75,
  Melbourne, Australia. Association for Computational Linguistics.

\bibitem[{Kudo and Richardson(2018)}]{kudo2018sentencepiece}
Taku Kudo and John Richardson. 2018.
\newblock \href {https://arxiv.org/pdf/1808.06226.pdf} {{SentencePiece}: A
  simple and language independent subword tokenizer and detokenizer for neural
  text processing}.
\newblock In \emph{Proceedings of the 2018 Conference on Empirical Methods in
  Natural Language Processing: System Demonstrations}, pages 66--71.

\bibitem[{Mikolov et~al.(2013)Mikolov, Chen, Corrado, and
  Dean}]{mikolov2013efficient}
Tom{\'{a}}s Mikolov, Kai Chen, Greg Corrado, and Jeffrey Dean. 2013.
\newblock \href {http://arxiv.org/abs/1301.3781} {Efficient estimation of word
  representations in vector space}.
\newblock In \emph{International Conference on Learning Representations,
  Workshop Track Proceedings}, Scottsdale, Arizona, USA.

\bibitem[{OpenAI et~al.(2023)OpenAI, Achiam, Adler, Agarwal, Ahmad, Akkaya,
  Aleman, Almeida, Altenschmidt, Altman, Anadkat, Avila, Babuschkin, Balaji,
  Balcom, Baltescu, Bao, Bavarian, Belgum, Bello, Berdine, Bernadett-Shapiro,
  Berner, Bogdonoff, Boiko, Boyd, Brakman, Brockman, Brooks, Brundage, Button,
  Cai, Campbell, Cann, Carey, Carlson, Carmichael, Chan, Chang, Chantzis, Chen,
  Chen, Chen, Chen, Chen, Chess, Cho, Chu, Chung, Cummings, Currier, Dai,
  Decareaux, Degry, Deutsch, Deville, Dhar, Dohan, Dowling, Dunning, Ecoffet,
  Eleti, Eloundou, Farhi, Fedus, Felix, Fishman, Forte, Fulford, Gao, Georges,
  Gibson, Goel, Gogineni, Goh, Gontijo-Lopes, Gordon, Grafstein, Gray, Greene,
  Gross, Gu, Guo, Hallacy, Han, Harris, He, Heaton, Heidecke, Hesse, Hickey,
  Hickey, Hoeschele, Houghton, Hsu, Hu, Hu, Huizinga, Jain, Jain, Jang, Jiang,
  Jiang, Jin, Jin, Jomoto, Jonn, Jun, Kaftan, Łukasz Kaiser, Kamali,
  Kanitscheider, Keskar, Khan, Kilpatrick, Kim, Kim, Kim, Kirchner, Kiros,
  Knight, Kokotajlo, Łukasz Kondraciuk, Kondrich, Konstantinidis, Kosic,
  Krueger, Kuo, Lampe, Lan, Lee, Leike, Leung, Levy, Li, Lim, Lin, Lin, Litwin,
  Lopez, Lowe, Lue, Makanju, Malfacini, Manning, Markov, Markovski, Martin,
  Mayer, Mayne, McGrew, McKinney, McLeavey, McMillan, McNeil, Medina, Mehta,
  Menick, Metz, Mishchenko, Mishkin, Monaco, Morikawa, Mossing, Mu, Murati,
  Murk, Mély, Nair, Nakano, Nayak, Neelakantan, Ngo, Noh, Ouyang, O'Keefe,
  Pachocki, Paino, Palermo, Pantuliano, Parascandolo, Parish, Parparita,
  Passos, Pavlov, Peng, Perelman, de~Avila Belbute~Peres, Petrov,
  de~Oliveira~Pinto, Michael, Pokorny, Pokrass, Pong, Powell, Power, Power,
  Proehl, Puri, Radford, Rae, Ramesh, Raymond, Real, Rimbach, Ross, Rotsted,
  Roussez, Ryder, Saltarelli, Sanders, Santurkar, Sastry, Schmidt, Schnurr,
  Schulman, Selsam, Sheppard, Sherbakov, Shieh, Shoker, Shyam, Sidor, Sigler,
  Simens, Sitkin, Slama, Sohl, Sokolowsky, Song, Staudacher, Such, Summers,
  Sutskever, Tang, Tezak, Thompson, Tillet, Tootoonchian, Tseng, Tuggle,
  Turley, Tworek, Uribe, Vallone, Vijayvergiya, Voss, Wainwright, Wang, Wang,
  Wang, Ward, Wei, Weinmann, Welihinda, Welinder, Weng, Weng, Wiethoff,
  Willner, Winter, Wolrich, Wong, Workman, Wu, Wu, Wu, Xiao, Xu, Yoo, Yu, Yuan,
  Zaremba, Zellers, Zhang, Zhang, Zhao, Zheng, Zhuang, Zhuk, and Zoph}]{gpt4}
OpenAI, Josh Achiam, Steven Adler, Sandhini Agarwal, Lama Ahmad, Ilge Akkaya,
  Florencia~Leoni Aleman, Diogo Almeida, Janko Altenschmidt, Sam Altman,
  Shyamal Anadkat, Red Avila, Igor Babuschkin, Suchir Balaji, Valerie Balcom,
  Paul Baltescu, Haiming Bao, Mohammad Bavarian, Jeff Belgum, Irwan Bello, Jake
  Berdine, Gabriel Bernadett-Shapiro, Christopher Berner, Lenny Bogdonoff, Oleg
  Boiko, Madelaine Boyd, Anna-Luisa Brakman, Greg Brockman, Tim Brooks, Miles
  Brundage, Kevin Button, Trevor Cai, Rosie Campbell, Andrew Cann, Brittany
  Carey, Chelsea Carlson, Rory Carmichael, Brooke Chan, Che Chang, Fotis
  Chantzis, Derek Chen, Sully Chen, Ruby Chen, Jason Chen, Mark Chen, Ben
  Chess, Chester Cho, Casey Chu, Hyung~Won Chung, Dave Cummings, Jeremiah
  Currier, Yunxing Dai, Cory Decareaux, Thomas Degry, Noah Deutsch, Damien
  Deville, Arka Dhar, David Dohan, Steve Dowling, Sheila Dunning, Adrien
  Ecoffet, Atty Eleti, Tyna Eloundou, David Farhi, Liam Fedus, Niko Felix,
  Simón~Posada Fishman, Juston Forte, Isabella Fulford, Leo Gao, Elie Georges,
  Christian Gibson, Vik Goel, Tarun Gogineni, Gabriel Goh, Rapha Gontijo-Lopes,
  Jonathan Gordon, Morgan Grafstein, Scott Gray, Ryan Greene, Joshua Gross,
  Shixiang~Shane Gu, Yufei Guo, Chris Hallacy, Jesse Han, Jeff Harris, Yuchen
  He, Mike Heaton, Johannes Heidecke, Chris Hesse, Alan Hickey, Wade Hickey,
  Peter Hoeschele, Brandon Houghton, Kenny Hsu, Shengli Hu, Xin Hu, Joost
  Huizinga, Shantanu Jain, Shawn Jain, Joanne Jang, Angela Jiang, Roger Jiang,
  Haozhun Jin, Denny Jin, Shino Jomoto, Billie Jonn, Heewoo Jun, Tomer Kaftan,
  Łukasz Kaiser, Ali Kamali, Ingmar Kanitscheider, Nitish~Shirish Keskar,
  Tabarak Khan, Logan Kilpatrick, Jong~Wook Kim, Christina Kim, Yongjik Kim,
  Jan~Hendrik Kirchner, Jamie Kiros, Matt Knight, Daniel Kokotajlo, Łukasz
  Kondraciuk, Andrew Kondrich, Aris Konstantinidis, Kyle Kosic, Gretchen
  Krueger, Vishal Kuo, Michael Lampe, Ikai Lan, Teddy Lee, Jan Leike, Jade
  Leung, Daniel Levy, Chak~Ming Li, Rachel Lim, Molly Lin, Stephanie Lin,
  Mateusz Litwin, Theresa Lopez, Ryan Lowe, Patricia Lue, Anna Makanju, Kim
  Malfacini, Sam Manning, Todor Markov, Yaniv Markovski, Bianca Martin, Katie
  Mayer, Andrew Mayne, Bob McGrew, Scott~Mayer McKinney, Christine McLeavey,
  Paul McMillan, Jake McNeil, David Medina, Aalok Mehta, Jacob Menick, Luke
  Metz, Andrey Mishchenko, Pamela Mishkin, Vinnie Monaco, Evan Morikawa, Daniel
  Mossing, Tong Mu, Mira Murati, Oleg Murk, David Mély, Ashvin Nair, Reiichiro
  Nakano, Rajeev Nayak, Arvind Neelakantan, Richard Ngo, Hyeonwoo Noh, Long
  Ouyang, Cullen O'Keefe, Jakub Pachocki, Alex Paino, Joe Palermo, Ashley
  Pantuliano, Giambattista Parascandolo, Joel Parish, Emy Parparita, Alex
  Passos, Mikhail Pavlov, Andrew Peng, Adam Perelman, Filipe de~Avila
  Belbute~Peres, Michael Petrov, Henrique~Ponde de~Oliveira~Pinto, Michael,
  Pokorny, Michelle Pokrass, Vitchyr~H. Pong, Tolly Powell, Alethea Power,
  Boris Power, Elizabeth Proehl, Raul Puri, Alec Radford, Jack Rae, Aditya
  Ramesh, Cameron Raymond, Francis Real, Kendra Rimbach, Carl Ross, Bob
  Rotsted, Henri Roussez, Nick Ryder, Mario Saltarelli, Ted Sanders, Shibani
  Santurkar, Girish Sastry, Heather Schmidt, David Schnurr, John Schulman,
  Daniel Selsam, Kyla Sheppard, Toki Sherbakov, Jessica Shieh, Sarah Shoker,
  Pranav Shyam, Szymon Sidor, Eric Sigler, Maddie Simens, Jordan Sitkin,
  Katarina Slama, Ian Sohl, Benjamin Sokolowsky, Yang Song, Natalie Staudacher,
  Felipe~Petroski Such, Natalie Summers, Ilya Sutskever, Jie Tang, Nikolas
  Tezak, Madeleine~B. Thompson, Phil Tillet, Amin Tootoonchian, Elizabeth
  Tseng, Preston Tuggle, Nick Turley, Jerry Tworek, Juan Felipe~Cerón Uribe,
  Andrea Vallone, Arun Vijayvergiya, Chelsea Voss, Carroll Wainwright,
  Justin~Jay Wang, Alvin Wang, Ben Wang, Jonathan Ward, Jason Wei, CJ~Weinmann,
  Akila Welihinda, Peter Welinder, Jiayi Weng, Lilian Weng, Matt Wiethoff, Dave
  Willner, Clemens Winter, Samuel Wolrich, Hannah Wong, Lauren Workman, Sherwin
  Wu, Jeff Wu, Michael Wu, Kai Xiao, Tao Xu, Sarah Yoo, Kevin Yu, Qiming Yuan,
  Wojciech Zaremba, Rowan Zellers, Chong Zhang, Marvin Zhang, Shengjia Zhao,
  Tianhao Zheng, Juntang Zhuang, William Zhuk, and Barret Zoph. 2023.
\newblock \href {https://openai.com/research/gpt-4} {{GPT}-4 technical report}.
\newblock \emph{{OpenAI} Blog}.

\bibitem[{Provilkov et~al.(2020)Provilkov, Emelianenko, and
  Voita}]{provilkov-etal-2020-bpe}
Ivan Provilkov, Dmitrii Emelianenko, and Elena Voita. 2020.
\newblock \href {https://doi.org/10.18653/v1/2020.acl-main.170} {{BPE}-dropout:
  Simple and effective subword regularization}.
\newblock In \emph{Proceedings of the 58th Annual Meeting of the Association
  for Computational Linguistics}, pages 1882--1892, Online. Association for
  Computational Linguistics.

\bibitem[{Radford et~al.(2019)Radford, Wu, Child, Luan, Amodei, and
  Sutskever}]{radford2019language}
Alec Radford, Jeffrey Wu, Rewon Child, David Luan, Dario Amodei, and Ilya
  Sutskever. 2019.
\newblock \href
  {https://cdn.openai.com/better-language-models/language_models_are_unsupervised_multitask_learners.pdf}
  {Language models are unsupervised multitask learners}.
\newblock \emph{OpenAI Blog}.

\bibitem[{Rozner et~al.(2021)Rozner, Potts, and
  Mahowald}]{rozner2021decrypting}
Joshua Rozner, Christopher Potts, and Kyle Mahowald. 2021.
\newblock \href {https://openreview.net/forum?id=Ah5CMODl52} {Decrypting
  cryptic crosswords: Semantically complex wordplay puzzles as a target for
  {NLP}}.
\newblock In \emph{Advances in Neural Information Processing Systems}.

\bibitem[{Rust et~al.(2023)Rust, Lotz, Bugliarello, Salesky, de~Lhoneux, and
  Elliott}]{rust2023language}
Phillip Rust, Jonas~F. Lotz, Emanuele Bugliarello, Elizabeth Salesky, Miryam
  de~Lhoneux, and Desmond Elliott. 2023.
\newblock \href {https://openreview.net/forum?id=FkSp8VW8RjH} {Language
  modelling with pixels}.
\newblock In \emph{International Conference on Learning Representations}.

\bibitem[{Sch{\"a}fer et~al.(2024)Sch{\"a}fer, Ravfogel, Hofmann, Pimentel, and
  Schlag}]{schafer2024language}
Anton Sch{\"a}fer, Shauli Ravfogel, Thomas Hofmann, Tiago Pimentel, and Imanol
  Schlag. 2024.
\newblock \href {https://arxiv.org/abs/2404.07982} {Language imbalance can
  boost cross-lingual generalisation}.
\newblock \emph{arXiv preprint arXiv:2404.07982}.

\bibitem[{Sennrich et~al.(2016)Sennrich, Haddow, and
  Birch}]{sennrich-etal-2016-neural}
Rico Sennrich, Barry Haddow, and Alexandra Birch. 2016.
\newblock \href {https://doi.org/10.18653/v1/P16-1162} {Neural machine
  translation of rare words with subword units}.
\newblock In \emph{Proceedings of the 54th Annual Meeting of the Association
  for Computational Linguistics (Volume 1: Long Papers)}, pages 1715--1725,
  Berlin, Germany. Association for Computational Linguistics.

\bibitem[{Stani{\'c} et~al.(2023)Stani{\'c}, Ashley, Serikov, Kirsch, Faccio,
  Schmidhuber, Hofmann, and Schlag}]{stanic2023languini}
Aleksandar Stani{\'c}, Dylan Ashley, Oleg Serikov, Louis Kirsch, Francesco
  Faccio, J{\"u}rgen Schmidhuber, Thomas Hofmann, and Imanol Schlag. 2023.
\newblock \href {https://arxiv.org/pdf/2309.11197.pdf} {The languini kitchen:
  Enabling language modelling research at different scales of compute}.
\newblock \emph{arXiv preprint arXiv:2309.11197}.

\bibitem[{Touvron et~al.(2023{\natexlab{a}})Touvron, Lavril, Izacard, Martinet,
  Lachaux, Lacroix, Rozière, Goyal, Hambro, Azhar, Rodriguez, Joulin, Grave,
  and Lample}]{touvron2023llama}
Hugo Touvron, Thibaut Lavril, Gautier Izacard, Xavier Martinet, Marie-Anne
  Lachaux, Timothée Lacroix, Baptiste Rozière, Naman Goyal, Eric Hambro,
  Faisal Azhar, Aurelien Rodriguez, Armand Joulin, Edouard Grave, and Guillaume
  Lample. 2023{\natexlab{a}}.
\newblock \href {https://arxiv.org/abs/2302.13971} {{LLaMA}: Open and efficient
  foundation language models}.
\newblock \emph{arXiv preprint arXiv:2302.13971}.

\bibitem[{Touvron et~al.(2023{\natexlab{b}})Touvron, Martin, Stone, Albert,
  Almahairi, Babaei, Bashlykov, Batra, Bhargava, Bhosale, Bikel, Blecher,
  Ferrer, Chen, Cucurull, Esiobu, Fernandes, Fu, Fu, Fuller, Gao, Goswami,
  Goyal, Hartshorn, Hosseini, Hou, Inan, Kardas, Kerkez, Khabsa, Kloumann,
  Korenev, Koura, Lachaux, Lavril, Lee, Liskovich, Lu, Mao, Martinet, Mihaylov,
  Mishra, Molybog, Nie, Poulton, Reizenstein, Rungta, Saladi, Schelten, Silva,
  Smith, Subramanian, Tan, Tang, Taylor, Williams, Kuan, Xu, Yan, Zarov, Zhang,
  Fan, Kambadur, Narang, Rodriguez, Stojnic, Edunov, and
  Scialom}]{touvron2023llama2}
Hugo Touvron, Louis Martin, Kevin Stone, Peter Albert, Amjad Almahairi, Yasmine
  Babaei, Nikolay Bashlykov, Soumya Batra, Prajjwal Bhargava, Shruti Bhosale,
  Dan Bikel, Lukas Blecher, Cristian~Canton Ferrer, Moya Chen, Guillem
  Cucurull, David Esiobu, Jude Fernandes, Jeremy Fu, Wenyin Fu, Brian Fuller,
  Cynthia Gao, Vedanuj Goswami, Naman Goyal, Anthony Hartshorn, Saghar
  Hosseini, Rui Hou, Hakan Inan, Marcin Kardas, Viktor Kerkez, Madian Khabsa,
  Isabel Kloumann, Artem Korenev, Punit~Singh Koura, Marie-Anne Lachaux,
  Thibaut Lavril, Jenya Lee, Diana Liskovich, Yinghai Lu, Yuning Mao, Xavier
  Martinet, Todor Mihaylov, Pushkar Mishra, Igor Molybog, Yixin Nie, Andrew
  Poulton, Jeremy Reizenstein, Rashi Rungta, Kalyan Saladi, Alan Schelten, Ruan
  Silva, Eric~Michael Smith, Ranjan Subramanian, Xiaoqing~Ellen Tan, Binh Tang,
  Ross Taylor, Adina Williams, Jian~Xiang Kuan, Puxin Xu, Zheng Yan, Iliyan
  Zarov, Yuchen Zhang, Angela Fan, Melanie Kambadur, Sharan Narang, Aurelien
  Rodriguez, Robert Stojnic, Sergey Edunov, and Thomas Scialom.
  2023{\natexlab{b}}.
\newblock \href {https://arxiv.org/abs/2307.09288} {Llama 2: Open foundation
  and fine-tuned chat models}.
\newblock \emph{arXiv preprint arXiv:2307.09288}.

\bibitem[{Wang et~al.(2019)Wang, Singh, Michael, Hill, Levy, and
  Bowman}]{wang2019glue}
Alex Wang, Amanpreet Singh, Julian Michael, Felix Hill, Omer Levy, and
  Samuel~R. Bowman. 2019.
\newblock \href {https://openreview.net/forum?id=rJ4km2R5t7} {{GLUE}: A
  multi-task benchmark and analysis platform for natural language
  understanding}.
\newblock In \emph{International Conference on Learning Representations}.

\bibitem[{Wu et~al.(2016)Wu, Schuster, Chen, Le, Norouzi, Macherey, Krikun,
  Cao, Gao, Macherey, Klingner, Shah, Johnson, Liu, Kaiser, Gouws, Kato, Kudo,
  Kazawa, Stevens, Kurian, Patil, Wang, Young, Smith, Riesa, Rudnick, Vinyals,
  Corrado, Hughes, and Dean}]{wu2016google}
Yonghui Wu, Mike Schuster, Zhifeng Chen, Quoc~V Le, Mohammad Norouzi, Wolfgang
  Macherey, Maxim Krikun, Yuan Cao, Qin Gao, Klaus Macherey, Jeff Klingner,
  Apurva Shah, Melvin Johnson, Xiaobing Liu, Łukasz Kaiser, Stephan Gouws,
  Yoshikiyo Kato, Taku Kudo, Hideto Kazawa, Keith Stevens, George Kurian,
  Nishant Patil, Wei Wang, Cliff Young, Jason Smith, Jason Riesa, Alex Rudnick,
  Oriol Vinyals, Greg Corrado, Macduff Hughes, and Jeffrey Dean. 2016.
\newblock \href {https://arxiv.org/abs/1609.08144} {Google's neural machine
  translation system: Bridging the gap between human and machine translation}.
\newblock \emph{arXiv preprint arXiv:1609.08144}.

\bibitem[{Xue et~al.(2022)Xue, Barua, Constant, Al-Rfou, Narang, Kale, Roberts,
  and Raffel}]{xue-etal-2022-byt5}
Linting Xue, Aditya Barua, Noah Constant, Rami Al-Rfou, Sharan Narang, Mihir
  Kale, Adam Roberts, and Colin Raffel. 2022.
\newblock \href {https://doi.org/10.1162/tacl_a_00461} {{B}y{T}5: Towards a
  token-free future with pre-trained byte-to-byte models}.
\newblock \emph{Transactions of the Association for Computational Linguistics},
  10:291--306.

\bibitem[{Yu et~al.(2023)Yu, Simig, Flaherty, Aghajanyan, Zettlemoyer, and
  Lewis}]{yu2023megabyte}
Lili Yu, Daniel Simig, Colin Flaherty, Armen Aghajanyan, Luke Zettlemoyer, and
  Mike Lewis. 2023.
\newblock \href {https://openreview.net/forum?id=JTmO2V9Xpz} {{MEGABYTE}:
  Predicting million-byte sequences with multiscale transformers}.
\newblock In \emph{Thirty-seventh Conference on Neural Information Processing
  Systems}.

\end{thebibliography}

\onecolumn
\appendix

\section{\texorpdfstring{\Cref{lemma:MI_similarity}}{Lemma 1} and Proof} \label{app:MI_similarity}

\newcommand{\mathcomment}[1]{\text{\color{gray} #1}}

\begin{definition}[Time-dependent Conditional Entropy and Mutual Information] 
We define a time-dependent conditional entropy as:
\begin{align}
    \ent(\Words_{<\T} \mid \Dubwords_{\leq \T}) \defeq \sum_{\words \in \wordsspace} \sum_{\dubwords \in \dubwordsspace} \phybrid(\words, \dubwords) \sum_{t=1}^{|\words|} \log \frac{1}{\phybrid(\words_{<t} \mid \dubwords_{\leq t})}
\end{align}
with $\ent(\Words_{<\T} \mid \Dubwords_{< \T})$ defined analogously. 
Note that, while $\phybrid(\words)$ stands for the probability of the single event $\{\words\}$, we write $\phybrid(\words_{<t})$ to represent the probability of the event composed by the union of all sequences with this prefix, i.e.,  $\cup_{\words' \in \wordsspace} \{\words_{<t} \circ \words'\}$, where $\circ$ stands for concatenation.

Accordingly, we define a time-dependent mutual information as:
\begin{align}
    \mi(\Words_{<\T}; \Dubwords_{\T} \mid \Dubwords_{<\T}) =
    \ent(\Words_{<\T} \mid \Dubwords_{<\T}) - \ent(\Words_{< \T} \mid \Dubwords_{\leq \T}).
\end{align}
\end{definition}

\begin{lemma} \label{lemma:MI_similarity}
    Let $\alphabetdupsmap: \alphabet \to \Sigmaduplicated$ be a deterministic function which maps duplicated subwords $\word \in \alphabet$ to their deduplicated versions $\dubword \in \Sigmaduplicated$.
    We then have that:
    \begin{align} \label{eq:app:MI_similarity}
        \ent_{\alphabetdupsmap}(\Words) &= \ent(\Dubwords) - \mi(\Words_{<\T}; \Dubwords_{\T} \mid \Dubwords_{< \T})
    \end{align}
\end{lemma}

\begin{proof} We start with the definition of $\ent_{\alphabetdupsmap}(\Words)$ and derive \Cref{lemma:MI_similarity}:
    \begin{subequations}
    \begin{align}
        \ent_{\alphabetdupsmap}(\Words) 
        &= \sum_{\words \in \wordsspace} p(\words) \sum_{t=1}^{|\words|} \log \frac{1}{\phybrid(\alphabetdupsmap(\word_t) \mid \words_{<t})} 
        \\
        &= \sum_{\words \in \wordsspace} \sum_{\dubwords \in \dubwordsspace} \phybrid(\words, \dubwords) \sum_{t=1}^{|\words|} \log \frac{1}{\phybrid(\dubword_t \mid \words_{<t})} 
        & \mathcomment{$\phybrid(\words, \dubwords) = p(\words)$ if $\alphabetdupsmap(\words) = \dubwords$ else $0$} \\
        &= \sum_{\words \in \wordsspace} \sum_{\dubwords \in \dubwordsspace} \phybrid(\words, \dubwords) \sum_{t=1}^{|\words|} \log \frac{1}{\phybrid(\dubword_t \mid \dubwords_{<t}, \words_{<t})}
        \!\!\!\!\!\!\!\!\!\!\!\!\!\!\!\! 
        & \mathcomment{$\phybrid(\dubwords_{<t} \mid \words_{<t})$ is deterministic} \\
        &= \sum_{\words \in \wordsspace} \sum_{\dubwords \in \dubwordsspace} \phybrid(\words, \dubwords) \sum_{t=1}^{|\words|} \log \frac{\phybrid(\words_{<t} \mid \dubwords_{< t})}{\phybrid(\dubword_t \mid \dubwords_{<t})\,\phybrid(\words_{<t} \mid \dubwords_{\leq t})}
        \!\!\!\!\!\!\!\!\!\!\!\!\!\!\!\!\!
        \!\!\!
        & \mathcomment{Bayes rule} \\
        &= \sum_{\words \in \wordsspace} \sum_{\dubwords \in \dubwordsspace} \phybrid(\words, \dubwords) \sum_{t=1}^{|\words|} \left(\log \frac{1}{p(\dubword_t \mid \dubwords_{<t})} + \log \frac{\phybrid(\words_{<t} \mid \dubwords_{< t})}{\phybrid(\words_{<t} \mid \dubwords_{\leq t})} \right)
        \!\!\!\!\!\!\!\!\!\!\!\!\!\!\!\!
        \!\!\!\!\!\!\!\!\!\!\!\!\!\!\!\!
        \!\!\!\!\!\!\!\!\!\!\!\!\!\!\!\!
        & \mathcomment{Split log} \\
        &= \sum_{\dubwords \in \dubwordsspace} p(\dubwords) \sum_{t=1}^{|\dubwords|} \log \frac{1}{p(\dubword_t \mid \dubwords_{<t})}
        + \sum_{\words \in \wordsspace} \sum_{\dubwords \in \dubwordsspace} \phybrid(\words, \dubwords) \sum_{t=1}^{|\words|} \log \frac{\phybrid(\words_{<t} \mid \dubwords_{< t})}{\phybrid(\words_{<t} \mid \dubwords_{\leq t})}
        \!\!\!\!\!\!\!\!\!\!\!\!\!\!\!\!
        \!\!\!\!\!\!\!\!\!\!\!\!\!\!\!\!
        \!\!\!\!\!\!\!\!\!\!\!\!\!\!\!\!
        \!\!\!\!\!\!\!\!\!\!\!\!\!\!\!\!
        \!\!\!\!\!\!\!\!\!\!\!\!\!\!\!\!
        & \mathcomment{Split sum} \label{eq:extra_justificiation}\\
        &= \ent(\Dubwords) + \sum_{\words \in \wordsspace} \sum_{\dubwords \in \dubwordsspace} \phybrid(\words, \dubwords) \sum_{t=1}^{|\words|} \log \frac{\phybrid(\words_{<t} \mid \dubwords_{< t})}{\phybrid(\words_{<t} \mid \dubwords_{\leq t})} 
        \!\!\!\!\!\!\!\!\!\!\!\!\!\!\!\!
        \!\!\!\!\!\!\!\!\!\!\!\!\!\!\!\!
        & \mathcomment{Definition of $\ent(\Dubwords)$} \\
        &= \ent(\Dubwords) + \ent(\Words_{<\T} \mid \Dubwords_{\leq \T}) - \ent(\Words_{<\T} \mid \Dubwords_{< \T})
        \!\!\!\!\!\!\!\!\!\!\!\!\!\!\!\!
        & \mathcomment{Definitions of conditional entropies} \\
        &= \ent(\Dubwords) - \mi(\Words_{<\T}; \Dubwords_{\T} \mid \Dubwords_{< \T})
        & \mathcomment{Definition of mutual information}
    \end{align}
    \end{subequations}
Note that $\phybrid(\words, \dubwords) > 0 \implies \alphabetdupsmap(\words) = \dubwords \implies |\words| = |\dubwords|$, which is required for arriving at \cref{eq:extra_justificiation}.
\end{proof}

\clearpage

\section{Near Duplicates in Modern LLMs}
\label{sec:llm_dup_stats}
\cref{tab:llm_dup_stats} shows the near duplicate rates in vocabularies of modern large language models, namely GPT-3.5 and GPT-4 models \cite{brown2020language, gpt4}, Claude 2.1 \cite{claude2}, Llama \cite{touvron2023llama, touvron2023llama2}, Mistral 7B and Mixtral 8x7B \cite{jiang2023mistral, jiang2024mixtral}, and Gemma \cite{gemma}.

\begin{table*}[h]
\centering
\begin{tabular}{lccccc}
\toprule
      & & \multicolumn{4}{c}{\textbf{Near Duplicate Rate}} \\ \cmidrule(lr){3-6}
\textbf{Model} & \textbf{Vocabulary Size} & $\whitespacedupsmap$ & $\lowerdupsmap$ & $\pluraldupsmap$ & $\alphabetdupsmap_\text{all}$\\
\midrule
GPT-\{3.5, 4, 4-turbo\} & 100k & 19\% & 24\% & 9\% & 43\%\\
Claude 2.1 & 65k & 25\% & 23\% & 9\% & 46\% \\
Llama 1 \& 2 & 32k & 17\% & 31\% & 22\% & 35\% \\
Mistral 7B \& Mixtral 8x7B & 32k & 15\% & 32\% & 23\% & 37\%  \\
Gemma 7B & 256k & 21\% & 20\% & 7\% & 39\% \\
\bottomrule
\end{tabular}
\caption{Near duplicate rates of modern LLM vocabularies. Computed as $1 - \frac{|\Sigmaduplicated|}{|\alphabet|} $, i.e., the ratio by which the size of the vocabulary decreases when mapping all subwords to their canonical versions using the respective $\alphabetdupsmap$. For the definitions of the deduplication mappings $\alphabetdupsmap$, see \cref{sec:deduplication_implementation}.}
\label{tab:llm_dup_stats}
\vspace{-7pt}
\end{table*}

\section{Near Duplicates under $\whitespacedupsmap$}
\label{app:space_near_dup}

Near duplicates under $\lowerdupsmap$ and $\pluraldupsmap$ are generally very close in meaning. $\lowerdupsmap$ pairs often consist of the lowercase and the capitalized version of a subword, where the latter might appear at the beginning of sentences. And (sub)word pairs under $\pluraldupsmap$ tend to refer to the same lemma. For $\whitespacedupsmap$, however, the contexts in which subword variants with and without a leading space appear might differ much more: in particular, we observe this with shorter subwords that can appear both as their own word (e.g., \subword{\_he} in `and he writes`) and as a substring of a longer unrelated word (e.g., \subword{he} in `breathe`).

To quantify how frequently $\whitespacedupsmap$ near duplicate pairs show such differences, we manually inspect 100 randomly sampled pairs and note whether the subwords convey comparable meaning in their first occurrences in the training data. We find that 53\% of the pairs carry highly similar meaning in both occurrences. Examples include use in compound words (e.g., \subword{\_writing} in isolation vs \subword{writing} in `handwriting', or \subword{\_held} in isolation vs \subword{held} in `long-held') or after special characters like newlines, parentheses, or quotation marks (e.g. \subword{\_wide} in isolation vs \subword{wide} in `(wide [...])'). Of the remaining 47\%, around half appear in contexts that are clearly different (e.g., \subword{he} and \subword{\_he} as in example above) while the rest are very short subwords where the conveyed meaning is hard to determine (e.g. \subword{\_sche} in `schemed' vs \subword{sche} in `Nietzsche').

\clearpage

\section{Effect of Added Parameters for Duplicates}
\label{app:embedding_params}

\newcommand{\projcontext}{\words_{<t}[\wordnumber{i}' \to \wordnumber{i}]}

A model $\model(\words)$ over duplicated subwords has more embedding parameters than its deduplicated counterpart $\model(\dubwords)$. 
In the perfect duplication setting, however, these extra parameters should not yield an unfair advantage. 
In particular, assume an equiprobable duplicates setting, where $p(\wordnumber{i} \mid \dubwordnumber{i}) = p(\wordnumber{i}' \mid \dubwordnumber{i}) = 0.5$.
We can show that, in this setting, if there exists an optimal model $\model(\words)$ (which achieves $\xent(\Words) = \ent(\Words)$), then there also exists an equivalent model $\model(\dubwords)$. 
Further, the input and output embeddings for duplicates $\wordnumber{i}$ and $\wordnumber{i}'$ are perfectly interchangeable;
the added embeddings thus provide no benefits, as the embedding of any $\wordnumber{i}'$ can be replaced with the embedding of the respective $\wordnumber{i}$ without affecting performance.

\newcommand{\emb}{\mathbf{v}}
\newcommand{\bias}{\mathrm{b}}
\newcommand{\hiddenstate}{\mathbf{h}}
\newcommand{\embnumber}[1]{\emb_{\scaleto{\circled{#1}}{7pt}}}
\newcommand{\biasnumber}[1]{\bias_{\scaleto{\circled{#1}}{7pt}}}

To see this, note that for the optimal $\model(\words)$, we have $\model(\words) = p(\words)$ almost everywhere. Consequently, by Bayes, $\model(\word \mid \words_{<t}) = p(\word \mid \words_{<t}) $.
We thus obtain
\begin{equation}
\model(\wordnumber{i} \mid \words_{<t}) = \model(\wordnumber{i}' \mid \words_{<t})
\end{equation}
which means that the model makes identical predictions for duplicate subwords. 
Predictions are thus not affected if we assign all $\wordnumber{i}'$ the output embedding of the respective $\wordnumber{i}$ (or vice-versa).

An analogous argument holds for input embeddings. Let $\projcontext$ denote the token sequence obtained when replacing every occurrence of $\wordnumber{i}'$ in $\words_{<t}$ with $\wordnumber{i}$. For the true distribution, we know that 
$
    p(\wordnumber{i} \mid \words_{<t}) = p(\wordnumber{i} \mid \projcontext)
$ due to the perfect equivalence, and thus also:
\begin{equation}
    \model(\wordnumber{i} \mid \words_{<t}) = \model(\wordnumber{i} \mid \projcontext)
\end{equation}
This means that $\model(\words)$ makes identical predictions, whether we replace $\wordnumber{i}'$ with  $\wordnumber{i}$ in the context or not. If we assign $\wordnumber{i}'$ the input embedding of $\wordnumber{i}$, model performance is not affected. By applying this argument iteratively for all $i$, we can show that all duplicates' input embeddings can be made identical.
If the input embeddings of all duplicates can be made identical, then there exists a model $\model(\dubwordnumber{i} \mid \dubwords_{<t})$ which, when given as input $\alphabetdupsmap(\words)$, outputs the same probability as $\model(\wordnumber{i} \mid \words_{<t}) + \model(\wordnumber{i}' \mid \words_{<t})$.
In particular, let $\emb_\word$ be the output embeddings associated with $\word$.
We can show that:
\begin{subequations}
\begin{align}
    \model(\wordnumber{i} \mid \words_{<t}) + \model(\wordnumber{i}' \mid \words_{<t}) 
    &= 2\,\model(\wordnumber{i} \mid \words_{<t}) \\ 
    &= 2\,\frac{\exp({\emb_{\wordnumber{i}}\,\hiddenstate })}{\sum\limits_{\word \in \alphabet} \exp({\emb_{\word}\,\hiddenstate})} \\
    &= 2\,\frac{\exp({\emb_{\wordnumber{i}}\,\hiddenstate })}{\sum\limits_{j = 1}^{|\Sigmaduplicated|} \exp({\emb_{\wordnumber{j}}\,\hiddenstate}) + \sum\limits_{j=1}^{|\Sigmaduplicated|} \exp({\emb_{\wordnumber{j}'}\,\hiddenstate})} \\
    &= \frac{2\,\exp({\emb_{\wordnumber{i}}\,\hiddenstate})}{2\sum\limits_{j=1}^{|\Sigmaduplicated|} \,\exp({\emb_{\wordnumber{j}}\,\hiddenstate})} \\
    &= \frac{\exp({\emb_{\dubwordnumber{i}}\,\hiddenstate})}{\sum\limits_{j=1}^{|\Sigmaduplicated|} \exp({\emb_{\dubwordnumber{j}}\,\hiddenstate})} \\
    &= \model(\dubwordnumber{i} \mid \dubwords_{<t})
\end{align}
\end{subequations}
So we can construct a $\model(\dubwords)$ model from $\model(\words)$ by simply making the embeddings of $\dubwordnumber{i}$ the same as $\wordnumber{i}$. The identical input embeddings preserve the value of the hidden state $\hiddenstate$ for any context, which, along with the identical output embeddings, ensures identical model outputs.

\clearpage
\section{Embedding Similarity Results}
\label{app:emb_similarity}
While in \cref{fig:dup_pdup_cossim}, it appears that a balance between duplicated and non-duplicated subwords (i.e. $\prduplicate = 1 - \prdontduplicate$ close to 0.5) leads to higher similarity of their embeddings, \cref{fig:dup_cossim_freq_pdup} shows that the underlying factor driving embedding alignment seems to be subword frequency, or more precisely, how often the rarer version of the subword (here always $\word'$ as $\prduplicate < 0.5$) occurs.

\begin{figure}[h]
    \centering
    \includegraphics[width=0.95\linewidth]{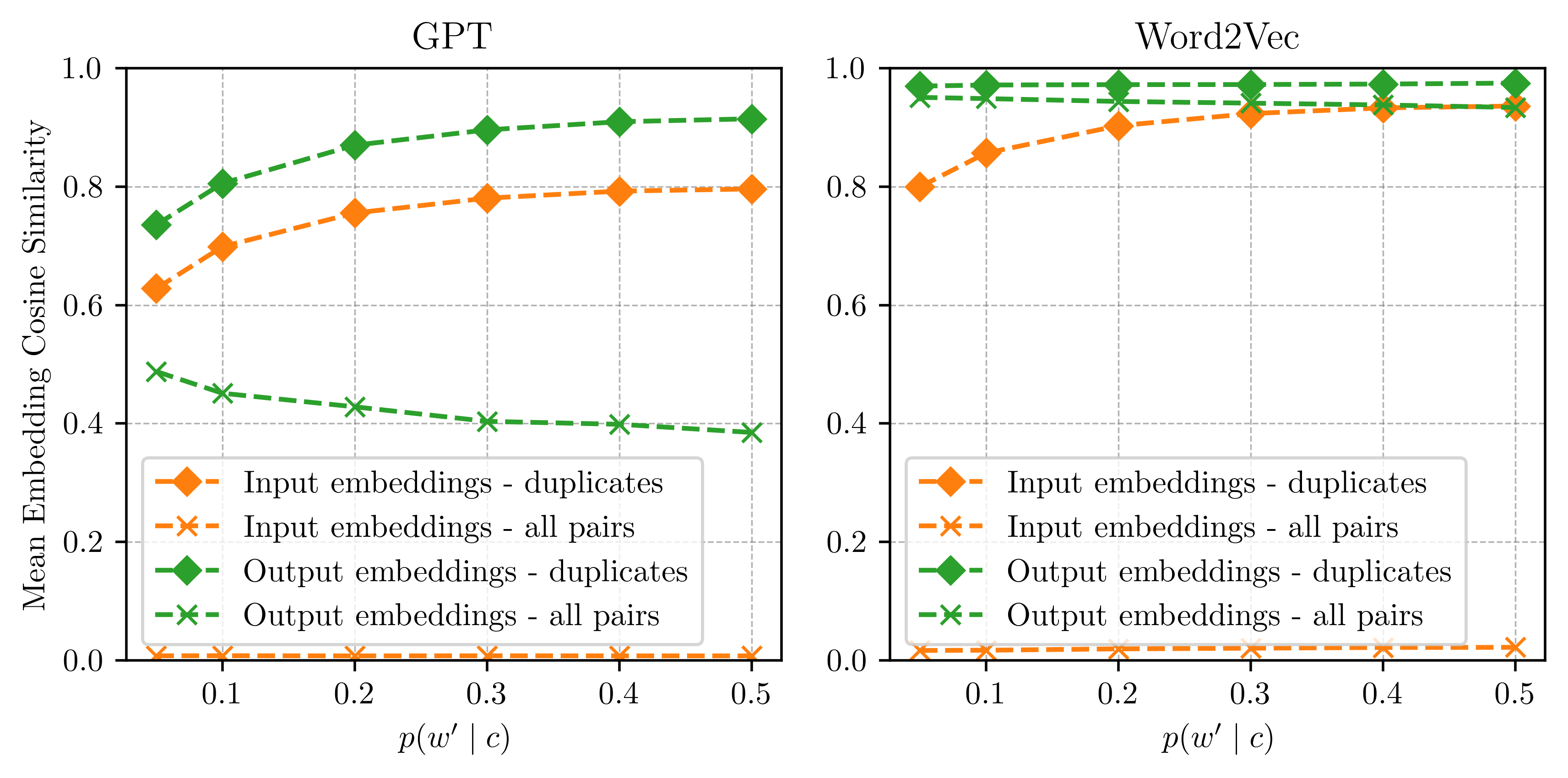}
    \caption{
    Embedding cosine similarity of duplicates $\wordnumber{i}$, $\wordnumber{i}'$ and random pairs $\wordnumber{i}$, $\wordnumber{j}$ to control for anisotropy. 
    Left: Our GPT model. Right: Word2vec embeddings trained on the same data (computed with Gensim).} %
    \label{fig:dup_pdup_cossim}
\end{figure}

\begin{figure}[h]
    \centering
    \includegraphics[width=0.5\linewidth]{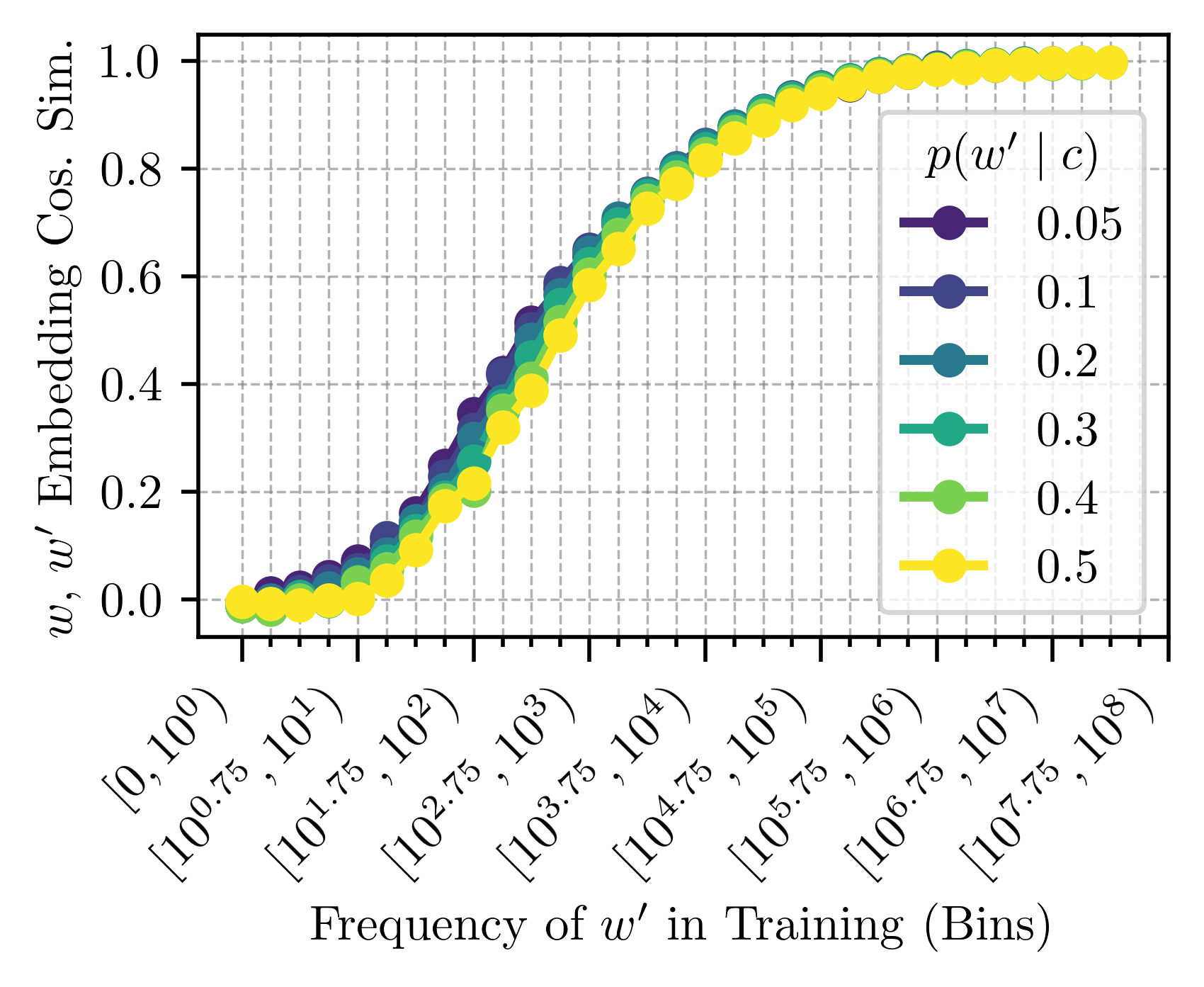}
    \vspace{-10pt}
    \caption{Embedding cosine similarity of duplicates $\wordnumber{i}$, $\wordnumber{i}'$, by frequency of the rarer $\wordnumber{i}'$. Frequencies binned and similarities averaged per bin.}
    \label{fig:dup_cossim_freq_pdup}
    \vspace{-9pt}
\end{figure}

\clearpage
\section{Effect of Deduplicated Subwords in the Context}
\label{app:dedup_subwords_in_context}
An increased number of deduplicated subwords in the context seems to generally lead to a steeper increase in surprisal than an increased number of non-deduplicated subwords, at least in the local context of 16 tokens (see Figure \cref{fig:dedup_input_causal_surprisal}). For the full context, the trend is less clear.
\begin{figure}[h]
    \centering
    \includegraphics[width=0.98\linewidth]{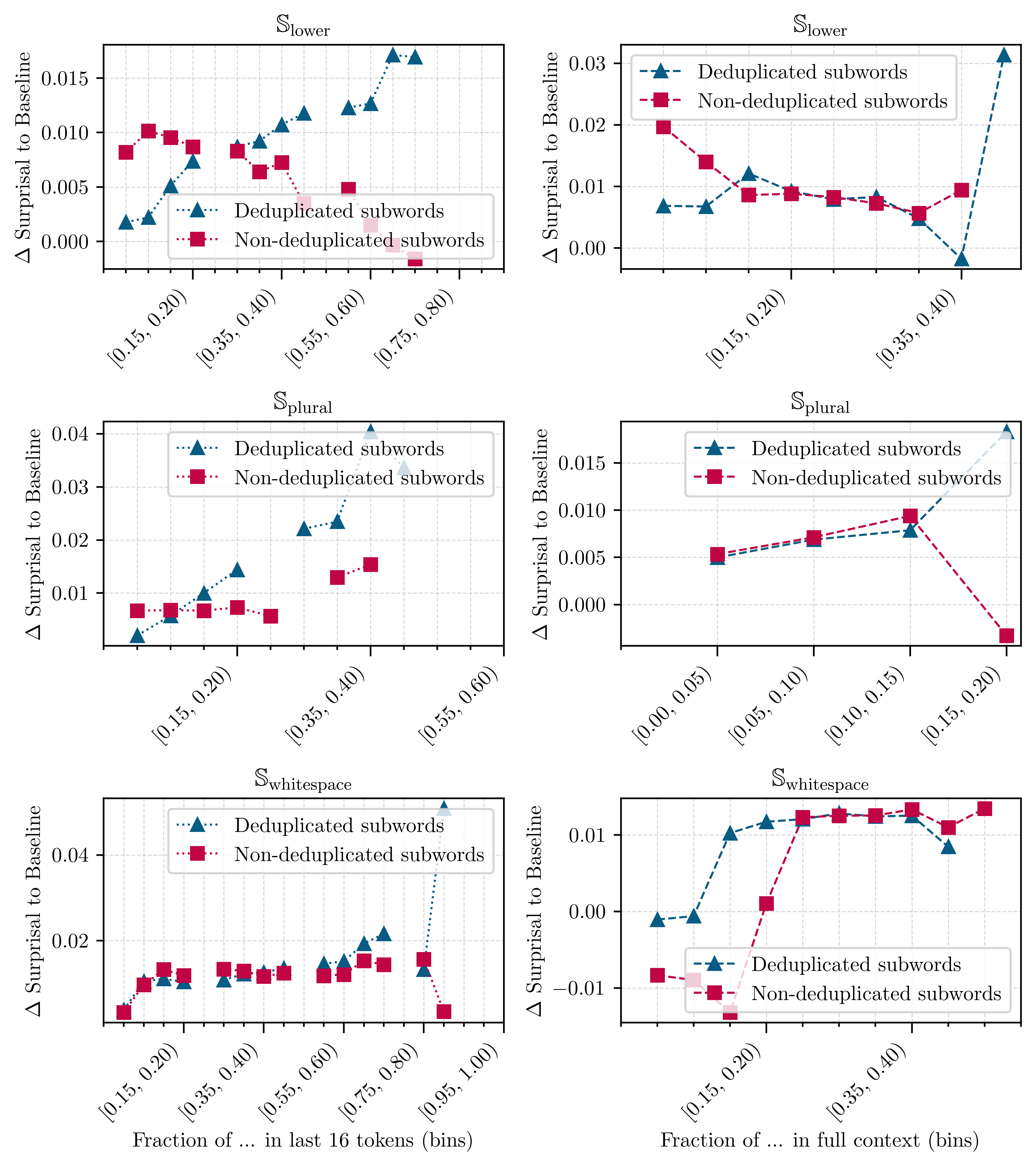}
    \caption{
        Deduplication of half of the vocabulary. Difference to baseline ($\model(\words)$) surprisal, depending on fraction of (non-)deduplicated subwords in context. Fractions binned and surprisal differences averaged per bin.
        To reduce noise and ensure readability, we only plot bins that contain at least 1k predicted tokens.
    }
    \label{fig:dedup_input_causal_surprisal}
\end{figure}

\end{document}